\documentclass[11pt]{article}

\usepackage[preprint]{acl}

\usepackage{amsmath,amsfonts,bm}

\def\eqref#1{equation~\ref{#1}}

\def\1{\bm{1}}

\DeclareMathAlphabet{\mathsfit}{\encodingdefault}{\sfdefault}{m}{sl}
\SetMathAlphabet{\mathsfit}{bold}{\encodingdefault}{\sfdefault}{bx}{n}

\DeclareMathOperator*{\argmax}{arg\,max}

\usepackage{times}
\usepackage{latexsym}

\usepackage[T1]{fontenc}

\usepackage[utf8]{inputenc}

\usepackage{microtype}

\usepackage{inconsolata}
\usepackage{amsmath}

\usepackage{graphicx}

\usepackage{hyperref}
\usepackage{url}
\usepackage{subcaption}
\usepackage{wrapfig}
\usepackage[T1]{fontenc}
\usepackage{booktabs}
\newcommand{\piref}{\pi_\text{ref}}
\newcommand{\pisft}{\pi^\text{SFT}} %
\newcommand{\cs}{\mathbf{s_t}}
\newcommand{\ns}{\mathbf{s_{t+1}}}
\newcommand{\ca}{\mathbf{a_t}}
\usepackage{mathtools}
\usepackage{amssymb}
\usepackage{amsthm}
\usepackage{soul}

\usepackage[algoruled,boxed,lined,noend]{algorithm2e}
\newtheorem{theorem}{Theorem}

\newtheorem{proposition}{Proposition}
\newtheorem{lemma}{Lemma}

\title{Failure Modes of Maximum Entropy RLHF}

\author{Ömer Veysel Çağatan \\
  KUIS AI Center, Koç University \\
  Sariyer/\.Istanbul, Turkey \\
  \texttt{ocagatan19@ku.edu.tr} \\ \And
  Barış Akgün \\
  Koç University, KUIS AI Center \\
  Sariyer/\.Istanbul, Turkey \\
  \texttt{baakgun@ku.edu.tr} \\}

\begin{document}
\maketitle
\begin{abstract}
In this paper, we show that Simple Preference Optimization (SimPO) can be derived as Maximum Entropy Reinforcement Learning, providing a theoretical foundation for this reference-free method. Motivated by SimPO's strong performance in offline preference optimization, we investigate whether Maximum Entropy RL can achieve similar results in online RLHF settings. Our experiments find that Maximum Entropy RL frequently exhibits overoptimization and unstable KL dynamics across model scales, with overoptimization persisting even at conservative learning rates for some configurations. Unlike KL-constrained methods that maintain stable training, entropy regularization fails to reliably prevent reward hacking and, in our experiments, correlates with the onset of overoptimization rather than guarding against it. Even in configurations where training remains stable, entropy regularization is not the stabilizing factor. Lastly, we discuss possible explanations for why SimPO succeeds in offline settings while Maximum Entropy RL struggles in online scenarios. Our findings suggest that reference-free approaches may face distinct challenges when applied to online versus offline preference learning.
\end{abstract}

\section{Introduction}

Aligning AI systems with human values is widely recognized as a central challenge in modern AI~\citep{bengio2025singaporeconsensusglobalai,stuart}. The dominant paradigm, Reinforcement Learning from Human Feedback (RLHF)~\citep{christiano2023deepreinforcementlearninghuman,stiennon2022learningsummarizehumanfeedback,ziegler2020finetuninglanguagemodelshuman,bai2022traininghelpfulharmlessassistant,ouyang2022traininglanguagemodelsfollow}, typically follows a three-stage pipeline: supervised fine-tuning, reward model training from preference data, and policy optimization with reinforcement learning under KL divergence regularization to constrain deviation from a reference model. While effective, this pipeline is computationally expensive and operationally complex, requiring separate reward models, substantial human annotation, and careful hyperparameter tuning to maintain stability.

These challenges have motivated direct alignment algorithms (DAAs)~\citep{rafailov2024scalinglawsrewardmodel}, which bypass explicit reward modeling and online RL. Direct Preference Optimization (DPO)~\citep{rafailov2024directpreferenceoptimizationlanguage} derives an analytical solution to a KL-regularized RL objective, expressing the reward implicitly as a function of the optimal policy and reducing preference learning to supervised optimization. More recently, Simple Preference Optimization (SimPO)~\citep{meng2024simposimplepreferenceoptimization} has demonstrated strong empirical performance while eliminating the reference model entirely, instead using length-normalized log-likelihoods and a fixed margin between preferred and dispreferred responses.

Despite its empirical success, SimPO has lacked a principled theoretical foundation comparable to that of reference-based methods such as DPO. In this work, we establish a connection between SimPO and Maximum Entropy Reinforcement Learning~\citep{ziebart2008maximum}. We show that SimPO can be interpreted as a closed-form solution to a Maximum Entropy RL objective, providing a theoretical grounding analogous to DPO's relationship with KL-constrained RL and suggesting that reference-free optimization can naturally arise from entropy regularization under appropriate conditions.

This perspective raises an important empirical question: if SimPO corresponds to an offline Maximum Entropy solution, can online Maximum Entropy RL serve as a viable alternative to KL-constrained methods in RLHF? To investigate this, we compare Maximum Entropy RL and KL-constrained RL on the TL;DR summarization benchmark~\citep{stiennon2022learningsummarizehumanfeedback} using models from the Pythia suite~\citep{biderman2023pythiasuiteanalyzinglarge}.

Our results reveal a clear asymmetry. While SimPO performs reliably in offline preference optimization, online Maximum Entropy RL frequently exhibits instability and overoptimization, even at conservative learning rates. We observe that increases in entropy often correlate with these instabilities, indicating that entropy regularization alone does not reliably prevent reward hacking and may, in some cases, exacerbate it. We hypothesize that SimPO benefits from implicit stabilizing mechanisms—such as dataset constraints and target margins—that partially substitute for the regularization provided by a reference model, whereas these protections are absent in online Maximum Entropy RL.

Our contributions are threefold. First, we provide a theoretical interpretation of SimPO as a Maximum Entropy Reinforcement Learning solution, situating it within established RL frameworks. Second, we empirically demonstrate that directly applying Maximum Entropy RL in online RLHF settings can lead to instability and overoptimization, highlighting limitations of entropy regularization in isolation. Third, we offer insight into why SimPO succeeds offline despite these challenges, emphasizing the role of implicit regularization through data constraints and margin-based objectives. Together, these findings clarify the relationship between entropy-based methods and preference optimization, and point to the need for additional regularization mechanisms for robust reference-free alignment in online settings.

\section{Background}
In this section, we review the relevant background topics, while additional related work is provided in Appendix~\ref{app:related_work}.
\subsection{Canonical RLHF}\label{section:prelims}

We reiterate the standard RLHF pipeline as outlined in \citep{ziegler2020finetuninglanguagemodelshuman} and subsequent works \citep{stiennon2022learningsummarizehumanfeedback, bai2022traininghelpfulharmlessassistant, ouyang2022traininglanguagemodelsfollow}. It consists of three main stages: (1) Supervised Fine-Tuning (SFT), (2) Reward Modeling, and (3) RL Optimization.

\paragraph{SFT:} A pre-trained LM is fine-tuned on task-specific high-quality data via supervised learning to obtain the initial policy $\pisft$.

\paragraph{Reward Modeling:} Prompts $x$ are sampled, and $\pisft$ generates answer pairs $(y_1, y_2)$. Human annotators indicate preferences $y_w \succ y_l \mid x$, assumed to reflect a latent reward function $r(x, y)$. A common approach is to model preferences with the Bradley-Terry (BT) model~\citep{19ff28b9-64f9-3656-ba40-08326a05748e}:
\begin{equation}\label{eq:bradley-terry}
p(y_1\succ y_2 \mid x)= \frac{e^{r^*(x, y_1)}}{e^{r^*(x, y_1)} + e^{r^*(x, y_2)}}
\end{equation}
Given a dataset $\mathcal{D}=\{x^{(i)}, y_w^{(i)}, y_l^{(i)}\}$, we learn a reward model $r_\phi$ by minimizing the binary classification loss:
\begin{align*}
\mathcal{L}_R &= -\mathbb{E}_{(x,y_w,y_l)\sim\mathcal{D}} \\
&\quad\left[\log \sigma\left(r_\phi(x, y_w) - r_\phi(x, y_l)\right)\right],
\end{align*}
where $\sigma$ is the sigmoid function. In practice, $r_\phi$ is initialized from $\pisft$ with a linear head, and reward outputs are normalized for stability.

\paragraph{RL Fine-Tuning:} Finally, the policy $\pi_\theta$ is optimized using the learned reward, constrained by a KL term to stay close to the reference policy $\piref = \pisft$:
\begin{align*}\label{eq:rl_objective}
\max_{\pi_\theta} \; \mathbb{E}_{x \sim \mathcal{D},\, y \sim \pi_\theta}\!\bigl[r_\phi(x, y)\bigr] \\
\;-\; \beta\, D_{\mathrm{KL}}\!\bigl[\pi_\theta(y \mid x)\,\|\,\piref(y \mid x)\bigr].
\end{align*}
This prevents overoptimization and distributional shift. In practice, this objective is optimized with PPO \citep{schulman2017proximalpolicyoptimizationalgorithms}, with the KL penalty folded into the per-step reward as $r(x, y) = r_\phi(x, y) - \beta(\log \pi_\theta(y \mid x) - \log \piref(y \mid x))$.

\subsection{Direct Preference Optimization}
Direct Preference Optimization (DPO)~\citep{rafailov2024directpreferenceoptimizationlanguage} has become a popular method for preference-based tuning. Unlike traditional approaches that train a separate reward model, DPO defines the reward directly in terms of the optimized policy:
\label{eq:dpo_reward}
\begin{align*}
r(x,y) = \beta \log \frac{\pi_\theta(y \mid x)}{\pi_{\text{ref}}(y \mid x)} + \beta \log Z(x),
\end{align*}

Here, $\pi_\theta$ is the current policy, $\pi_{\text{ref}}$ is a reference (often the SFT model), and $Z(x)$ is a normalization term.
DPO incorporates this reward into the Bradley-Terry~\citep{19ff28b9-64f9-3656-ba40-08326a05748e} framework, where preference probabilities are given by:$p(y_w \succ y_l \mid x) = \sigma \left( r(x, y_w) - r(x, y_l) \right)$. This leads to the following objective, computed over preference triplets $(x, y_w, y_l)$:

\begin{align*}
\mathcal{L}_{\text{DPO}}(\pi_\theta; \pi_{\text{ref}}) = - \mathbb{E}_{(x, y_w, y_l) \sim \mathcal{D}} \\\left[ \log \sigma  \left( \beta \log \frac{\pi_\theta(y_w \mid x)}{\pi_{\text{ref}}(y_w \mid x)} - \beta \log \frac{\pi_\theta(y_l \mid x)}{\pi_{\text{ref}}(y_l \mid x)}\right) \right]
\end{align*}\label{eq:dpo}

By modeling preferences directly through policy ratios, DPO removes the need for an explicit reward model while remaining grounded in a probabilistic preference framework.

\subsection{Simple Preference Optimization}

Simple Preference Optimization (SimPO)~\citep{meng2024simposimplepreferenceoptimization} is a reference-free method for preference-based fine-tuning that aligns the reward used in training with the likelihood used at inference. Unlike DPO, SimPO eliminates the need for a reference policy by defining the reward as the length-normalized log-likelihood of the model output:

\begin{align*}
r_{\text{SimPO}}(x,y) = \frac{\beta}{|y|} \log \pi_\theta(y \mid x) = \\ \frac{\beta}{|y|} \sum_{i=1}^{|y|} \log \pi_\theta(y_i \mid x, y_{<i})
\end{align*}

This formulation ensures that the reward ranking $r(x, y_w) > r(x, y_l)$ aligns with the generation-time likelihood ranking $p_\theta(y_w \mid x) > p_\theta(y_l \mid x)$, which is often violated in DPO. SimPO also introduces a target margin $\gamma > 0$ into the Bradley-Terry model to encourage separation between preferred and dispreferred responses:

\begin{align*}
p(y_w \succ y_l \mid x) = \sigma \left( r(x, y_w) - r(x, y_l) - \gamma \right)
\end{align*}

This leads to the SimPO training objective:
\label{eq:simpo}
\begin{align*}
\mathcal{L}_{\text{SimPO}}(\pi_\theta) = - \mathbb{E} \left[ \log \sigma \left( \frac{\beta}{|y_w|} \log \pi_\theta(y_w \mid x) \right. \right. \\
\left. \left. \quad - \frac{\beta}{|y_l|} \log \pi_\theta(y_l \mid x) - \gamma \right) \right]
\end{align*}

\section{SimPO is the Maximum Entropy RL}
SimPO is a widely used preference alignment method, appreciated for its strong empirical performance and simplicity due to its reference-free objective. However, it lacks a theoretical foundation, unlike reference-based approaches such as DPO, which is derived from a KL-constrained RL objective. Recent work~\citep{liu2024understandingreferencepoliciesdirect} made the important observation that posterior probability rewards correspond to Maximum Entropy RL in their analysis of reference policies. Building on this insight, we establish the connection between this Maximum Entropy RL formulation and SimPO, showing that SimPO can be understood as Maximum Entropy RL with the addition of length normalization and target margins from preference learning.

\subsection{Maximum Entropy RL}
Maximum Entropy Reinforcement Learning (MaxEnt RL) augments the standard RL objective with an entropy term, encouraging policies that align with the soft value function~\citep{ziebart2008maximum, toussaint2009robot, rawlik2013stochastic, fox2015taming, o2016combining, abdolmaleki2018maximum, sac, mazoure2020leveraging, han2021max,zhang2025maximumentropymisleadspolicy}. It is deeply connected to probabilistic inference~\citep{toussaint2009robot, rawlik2013stochastic, levine2018reinforcement} and supported by both stochastic inference~\citep{ziebart2010modeling, eysenbach2021maximum} and game-theoretic foundations~\citep{grunwald2004game, ziebart2010maximum, han2021max, kim2023adaptive}. MaxEnt is often favored for promoting exploration~\citep{sac, hazan2019provably}, smoothing optimization~\citep{ahmed2019understanding}, and enabling robust decision-making~\citep{eysenbach2021maximum}.

The general form of the Maximum Entropy Reinforcement Learning (MaxEnt RL) objective can be written as
\begin{align*}
\pi^\star = \argmax_{\pi} \; \mathbb{E}_{\tau \sim p^\pi(\tau)} \\ \left[ \sum_{t=1}^T r(\cs, \ca) + \alpha \, \mathcal{H}_\pi[\ca \mid \cs] \right]
\end{align*}

where $\tau = (\mathbf{s_1}, \mathbf{a_1}, \mathbf{s_2}, \mathbf{a_2}, \dots, \mathbf{s_T}, \mathbf{a_T})$ is a trajectory sampled under policy $\pi$, and $p^\pi(\tau) = p_1(\mathbf{s_1}) \prod_{t=1}^T \pi(\ca \mid \cs) \, p(\ns \mid \cs, \ca)$ denotes the trajectory distribution induced by $\pi$. The term $\mathcal{H}_\pi[\ca \mid \cs] = -\int \pi(\ca \mid \cs) \log \pi(\ca \mid \cs) \, d\ca$ represents the conditional entropy of the policy at each time step, and the temperature coefficient $\alpha$ controls the trade-off between reward maximization and policy stochasticity.

\subsection{SimPO from Maximum Entropy RL}
RLHF is commonly modeled as a contextual bandit problem, though some approaches treat it as a token-level MDP~\citep{rafailov2024rqlanguagemodel,xie2024exploratorypreferenceoptimizationharnessing}. In this work, we adopt the contextual bandit view~\citep{Elwood_2023}, under which the maximum entropy formulation aligns with KL-constrained objectives. The resulting objective is given as follows.

\begin{equation}\label{equat:maxentrl}
\max_{\pi} \mathbb{E}_{x\sim D,y\sim\pi} [r(x, y)] + \alpha \mathcal{H}[\pi(y|x)]
\end{equation}

It is straightforward to show that the optimal policy of equation~\ref{equat:maxentrl} (proof in Appendix~\ref{app:mathderivations}) is as follows:
\begin{equation}
    \pi_r(y|x) = \frac{1}{Z(x)} \exp\left(\frac{1}{\alpha} r(x, y)\right)
\end{equation}

Following the analytical approach used in DPO's derivation, we can rearrange this optimal policy equation to express the reward function in terms of the policy:

\begin{equation}
r(x, y) = \alpha \log \pi_r(y|x) + \alpha \log Z(x)
\end{equation}

Now, applying this reparameterization to the Bradley-Terry preference model. For the ground-truth reward $r^*$ and corresponding optimal policy $\pi^*$, the preference probability becomes:

\begin{equation}
p^*(y_1 \succ y_2 | x) = \sigma(r^*(x, y_1) - r^*(x, y_2))
\end{equation}

Substituting our reparameterization:

\begin{align*}
p^*(y_1 &\succ y_2 | x) \nonumber \\
&= \sigma\big(\alpha \log \pi^*(y_1|x) + \alpha \log Z(x) \nonumber \\
&\qquad - \alpha \log \pi^*(y_2|x) - \alpha \log Z(x)\big) \\
&= \sigma\left(\alpha \log \pi^*(y_1|x) - \alpha \log \pi^*(y_2|x)\right)
\end{align*}

Crucially, the partition function $Z(x)$ cancels out, eliminating the need to compute it explicitly. This establishes a direct connection between the Bradley--Terry preference model and the optimal policy induced by Maximum Entropy RL, and constitutes the core theoretical contribution of this section: \emph{reference-free optimization arises naturally from the partition-function cancellation, independent of any further design choices}.

To recover the SimPO training objective used in practice, we introduce two additional components. We emphasize that these components are \emph{not} consequences of the Maximum Entropy RL derivation; rather, they are post-hoc empirical modifications motivated by structural considerations in preference learning, analogous to how length-normalized DPO~\citep{meng2024simposimplepreferenceoptimization} adds length normalization on top of the standard KL-constrained derivation as a practical fix.

First, we apply length normalization by scaling the temperature with the response length. This is necessary in SimPO because, unlike KL-constrained methods such as DPO where the reference policy induces implicit regularization through cancellation of log-probabilities, SimPO optimizes raw log-likelihoods and is therefore prone to severe length exploitation without explicit normalization. Second, inspired by ${\psi}PO$~\citep{azar2023generaltheoreticalparadigmunderstand}, we incorporate a target reward margin $\gamma > 0$ as a fixed substitute for the adaptive margin implicitly provided by the reference policy in DPO. These additions shape the optimization landscape and improve empirical stability, but do not alter the underlying connection between SimPO and Maximum Entropy RL established above. This leads to the SimPO objective for a parametric policy $\pi_\theta$:

\begin{align*}
\mathcal{L}_{\text{SimPO}}&(\pi_\theta) = -\mathbb{E}_{(x,y_w,y_l)\sim \mathcal{D}} \bigg[ \log \sigma\bigg( \nonumber \\
&\quad \frac{\alpha}{|y_w|} \log \pi_\theta(y_w|x) \nonumber \\
&\quad - \frac{\alpha}{|y_l|} \log \pi_\theta(y_l|x) - \gamma\bigg) \bigg]
\end{align*}
This derivation reveals that SimPO is equivalent to Maximum Entropy RL under the contextual bandit formulation, augmented with length normalization and a target margin, making explicit the theoretical connection that underlies SimPO's design.

\begin{figure*}[t]
    \centering
    \includegraphics[width=0.8\textwidth]{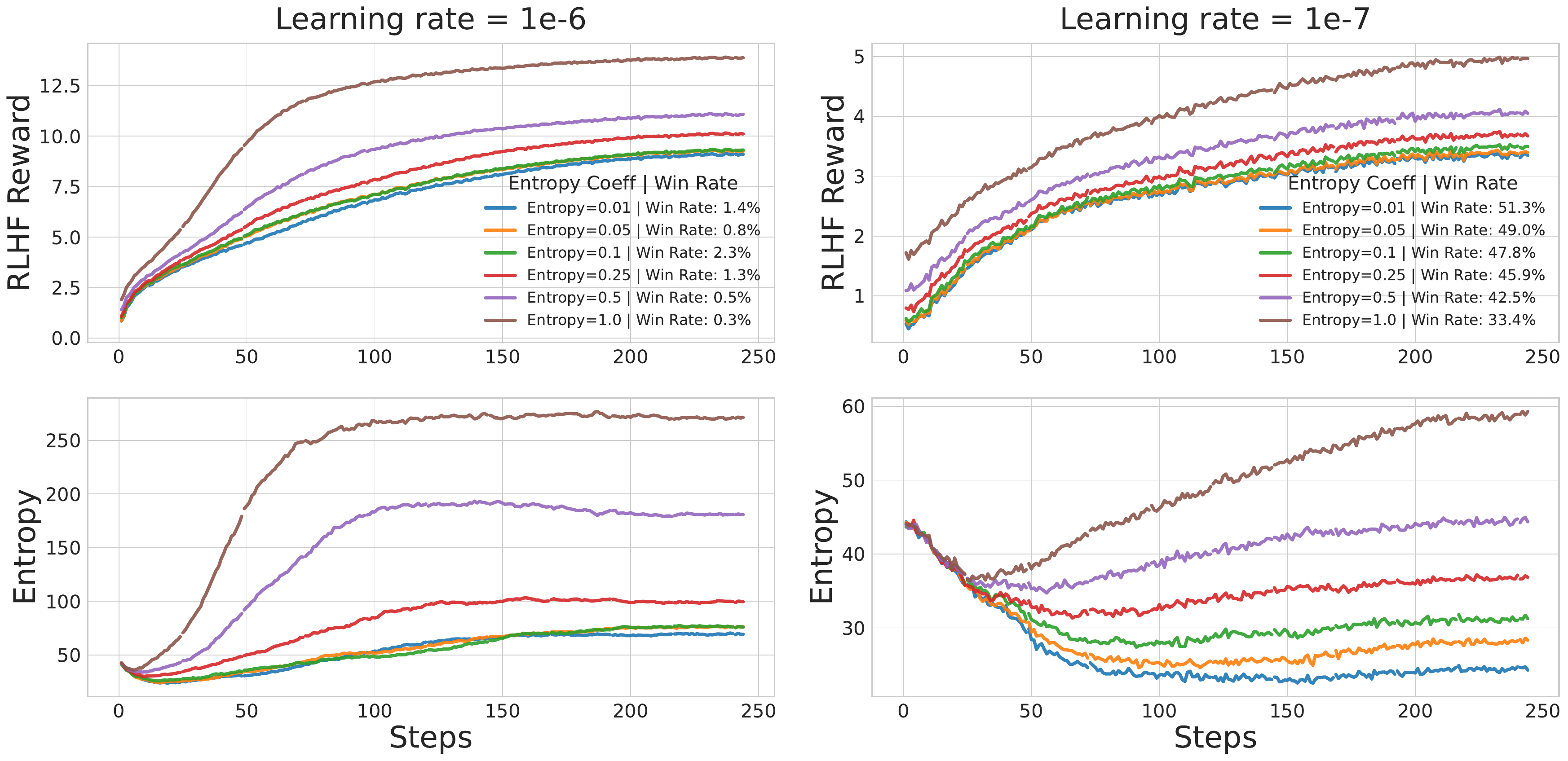}
    \caption{RLHF reward and entropy bonus during training for Pythia 1B with different entropy coefficients at learning rates 1e-6 (left) and 1e-7 (right). Win rates are reported in the legend for each entropy bonus coefficient setting.}
    \label{fig:pythia1b}
\end{figure*}

\paragraph{Theoretical Guarantees.} Following the same theoretical framework as DPO, SimPO inherits analogous guarantees regarding representational completeness, equivalence class preservation, and consistency under the Bradley-Terry preference model. The detailed proofs and formal statements of these properties are provided in Appendix~\ref{app:mathderivations}.

\section{Maximum Entropy RLHF}
Having established the theoretical connection between SimPO and Maximum Entropy RL, we now turn to the online RLHF setting. Our goal is to evaluate whether Maximum Entropy RL can perform comparably to its KL-constrained counterpart when applied directly to preference optimization.

\subsection{Experimental Setup and Methodology}

We conduct experiments using 1B, 2.8B, and 6.9B parameter models from the Pythia suite~\citep{biderman2023pythiasuiteanalyzinglarge}, trained with RLOO~\citep{ahmadian2024basicsrevisitingreinforcestyle} on the TL;DR summarization dataset~\citep{stiennon2022learningsummarizehumanfeedback}. All experiments follow the RLHF training recipe described by \citet{huang2024nimplementationdetailsrlhf} and are implemented using the TRL library~\citep{vonwerra2022trl}. Model alignment is evaluated using simulated preference win rates computed with GPT-4o-mini~\citep{openai2024gpt4ocard} as a proxy evaluator, measured against reference summaries for TL;DR using greedy decoding unless otherwise stated. Our experimental protocol closely follows the setup of \citet{rafailov2024scalinglawsrewardmodel}, enabling direct comparison with prior work. Complete hyperparameter settings are provided in Appendix~\ref{app:hyperparameters}.

Our aim is to study overoptimization phenomena in online RLHF, which are difficult to quantify reliably in more complex training setups involving larger models, diverse tasks, or heterogeneous reward signals. The TL;DR benchmark with Pythia models provides a controlled setting in which overoptimization can be measured consistently through preference win rates against fixed reference summaries, allowing systematic comparison across optimization objectives and model scales.

We adopt RLOO as a critic-free alternative to the standard RLHF pipeline while optimizing equivalent reward objectives. In the KL-constrained formulation, the reward is defined as
\begin{align*}
r(x, y) = r_\phi(x, y) \\
- \beta \Big( \log \pi_\theta(y|x) - \log \pi_{\text{ref}}(y|x) \Big),
\end{align*}
whereas in the Maximum Entropy formulation, the reward takes the form
\begin{equation}\label{eq:maxent_reward_unnormalized}
r(x, y) = r_\phi(x, y) - \alpha \log \pi_\theta(y|x).
\end{equation}

\begin{figure*}[t]
    \centering
    \includegraphics[width=0.8\textwidth]{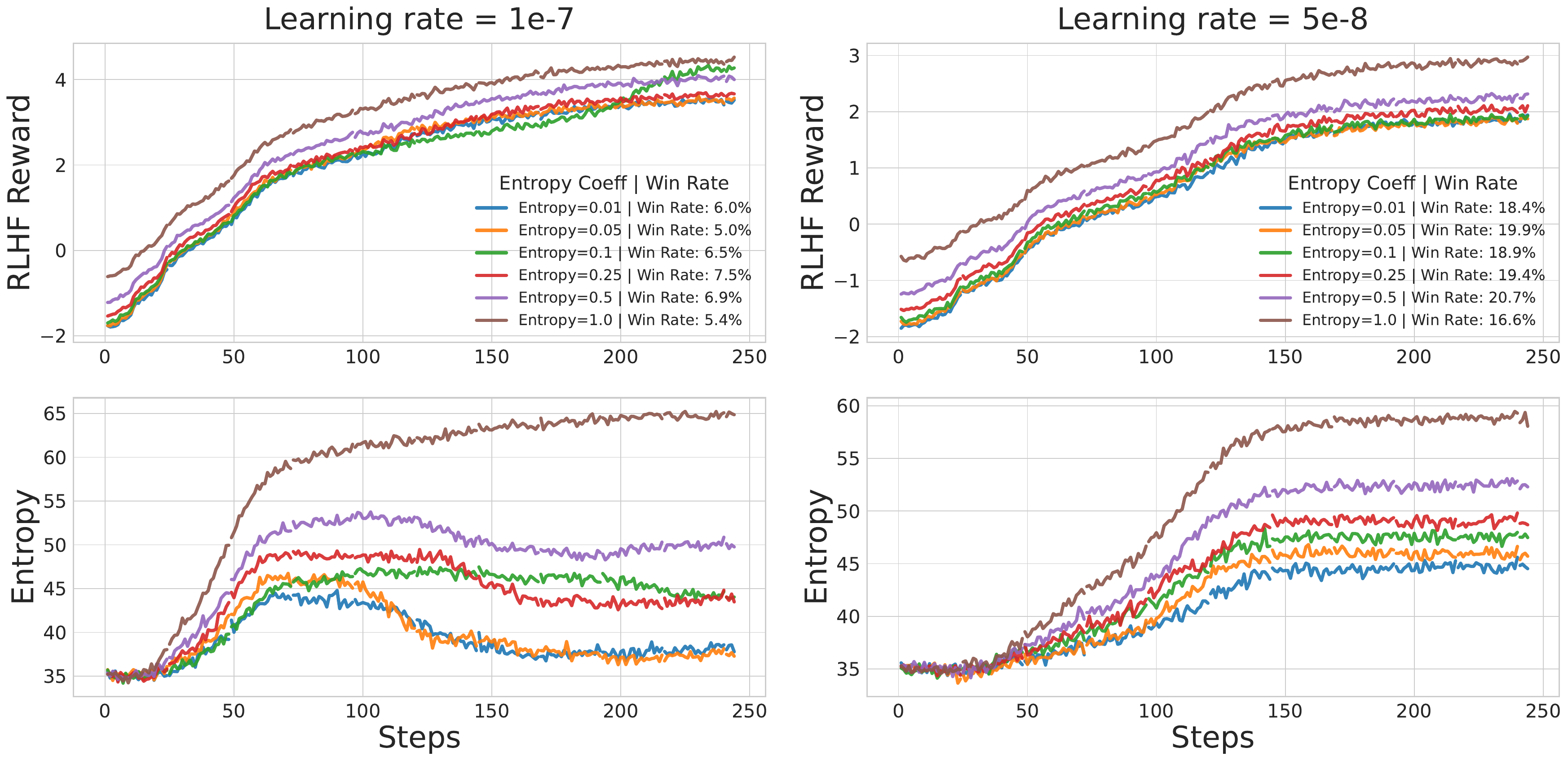}
    \caption{RLHF reward and entropy bonus during training for Pythia 2.8B with different entropy bonus coefficients at learning rates 1e-7 (left) and 5e-8 (right). Win rates are reported in the legend for each entropy bonus coefficient setting.}
    \label{fig:pythia2.8b}
\end{figure*}

When applying Maximum Entropy RL to sequence generation, the entropy term can be trivially increased by producing longer responses. Unlike KL-constrained objectives, which include an explicit reference policy that naturally counteracts such length effects, the reference-free formulation lacks an inherent mechanism to penalize verbosity. We therefore employ length normalization in the entropy term to ensure that the reward contribution is comparable across responses of different lengths and to prevent systematic length exploitation during optimization. This yields the reward
\begin{equation}\label{eq:maxent_reward_normalized}
r(x, y) = r_\phi(x, y) - \frac{\beta}{|y|} \log \pi_\theta(y|x).
\end{equation}

\paragraph{Notation bridge.} The temperature coefficient $\alpha$ used in the Maximum Entropy derivations of Section~3 (Equations~\ref{equat:maxentrl}--\ref{eq:maxent_reward_unnormalized}) corresponds to $\beta$ once length normalization is introduced (Equation~\ref{eq:maxent_reward_normalized}); the ``Entropy Coeff'' values in our figures correspond to $\beta$. We use $\beta$ throughout for consistency with the SimPO and RLHF literature.

\paragraph{Implementation of entropy regularization.} We implement entropy regularization as reward shaping (Equation~\ref{eq:maxent_reward_normalized}), \emph{not} as an additional PPO entropy bonus. This mirrors the standard KL-constrained RLHF implementation~\citep{huang2024nimplementationdetailsrlhf,vonwerra2022trl}, where the KL penalty is similarly folded into the per-step reward rather than applied as a separate loss term. KL-constrained and Maximum Entropy runs therefore share the same optimizer, hyperparameter regime, and shaping mechanism, differing only in whether the shaping term references an external policy.

We also attempted to extend evaluation to Anthropic-HH~\citep{bai2022traininghelpfulharmlessassistant}. However, we were unable to reproduce DPO baselines from the official repository, and the TRL implementation exhibited similar instability, with DPO failing to outperform SFT under high run-to-run variance. We therefore restrict our analysis to TL;DR.

\begin{figure*}[t]
    \centering
    
    \includegraphics[width=0.8\textwidth]{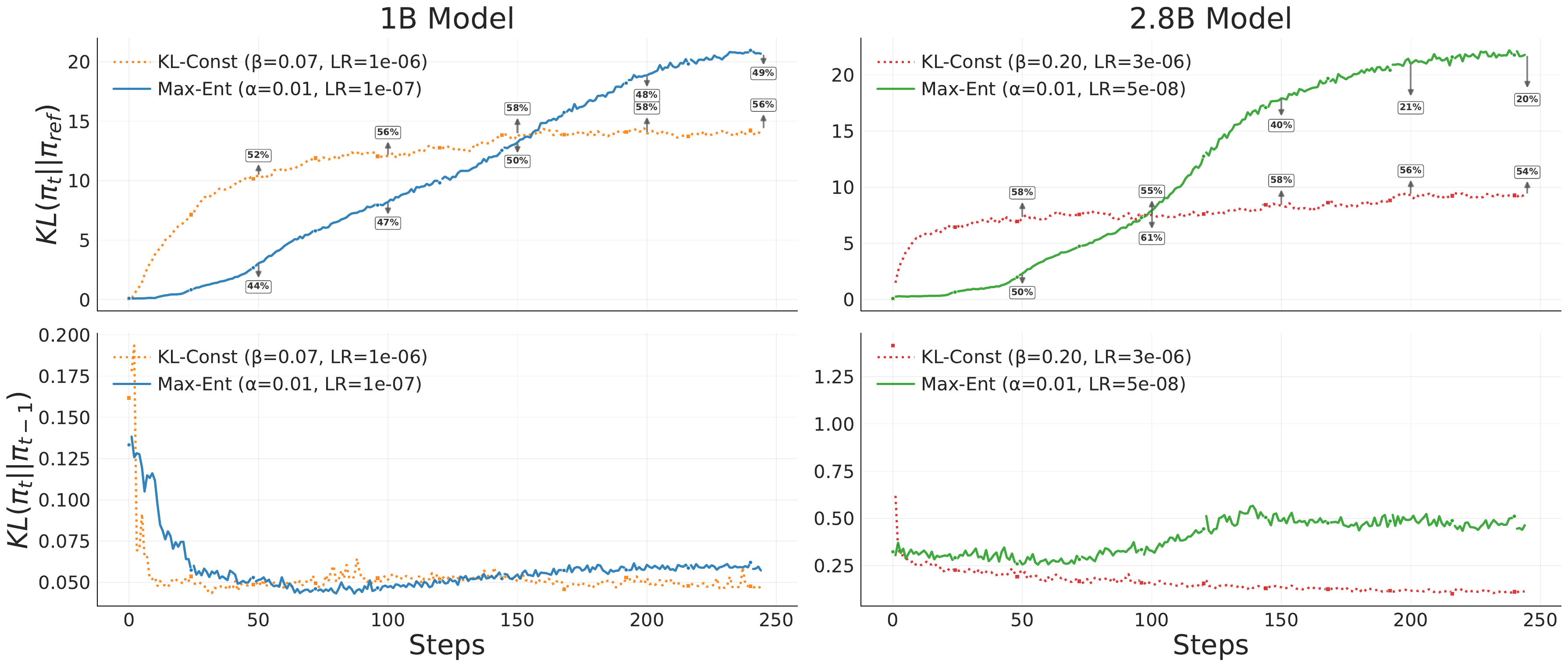}
    \caption{KL divergence metrics for KL-constrained and Maximum Entropy RL across training. 
    Top row shows KL divergence between the current policy and the reference policy ($KL(\pi_t \Vert \pi_{\text{ref}})$) 
    for Pythia 1B (left) and 2.8B (right). 
    Bottom row shows KL divergence between consecutive policy iterations ($KL(\pi_t \Vert \pi_{t-1})$).}
    \label{fig:conseckl}
\end{figure*}
\subsection{Results and Analysis}

\paragraph{Online Maximum Entropy RLHF across model scales.}

We evaluate Maximum Entropy RL in an online RLHF setting using Pythia models at three scales: 1B, 2.8B, and 6.9B parameters. Learning rates are chosen based on prior work, with $1 \times 10^{-6}$ corresponding to a regime where KL-constrained RLHF performs well, and smaller learning rates following common practice in reference-free methods. Entropy coefficients are selected via grid search; the full sweep is reported in Appendix~\ref{app:hyperparameters}.

Across all model scales, Maximum Entropy RL is highly sensitive to the learning rate. At $1 \times 10^{-6}$, all models rapidly enter overoptimized regimes regardless of the entropy coefficient. Lower learning rates improve stability, but do not fully eliminate this behavior.

For Pythia 1B, reducing the learning rate to $1 \times 10^{-7}$ yields stable training and improved win rates relative to the SFT baseline (Figure~\ref{fig:pythia1b}). However, these gains are not driven by entropy regularization: similar performance is achieved even when the entropy coefficient is set to zero. Moreover, stable runs exhibit decaying entropy bonuses, while overoptimized runs show increasing entropy bonuses, suggesting that entropy correlates with reward hacking rather than preventing it---a correlational claim, with converging evidence in Figure~\ref{fig:conseckl}, Section~\ref{sec:minentropy}, and the PPO-clipping controls below.

\paragraph{Scaling behavior.}

For Pythia 2.8B, we restrict experiments to learning rates of $1 \times 10^{-7}$ and $5 \times 10^{-8}$, as higher learning rates consistently caused rapid overoptimization. Despite these conservative settings, all Maximum Entropy variants overoptimize and fail to outperform the SFT baseline (Figure~\ref{fig:pythia2.8b}). In contrast, KL-constrained RLHF remains stable and effective at a higher learning rate of $1 \times 10^{-6}$.

Interestingly, Pythia~6.9B does not exhibit behavior intermediate between smaller and larger models; instead, its training dynamics more closely resemble those of Pythia~1B than Pythia~2.8B. Training remains stable at a learning rate of $1 \times 10^{-7}$ (Figure~\ref{fig:pythia6.9b-1e-7}), but increasing the learning rate to $1 \times 10^{-6}$ (Figure~\ref{fig:pythia6.9B-1e-6-resuts}) leads to overoptimization, mirroring the instability observed in smaller models. However, KL-constrained RLHF remains stable for Pythia~6.9B while consuming only a small KL budget (Figure~\ref{fig:pythia6.9b-kl}), indicating that effective KL budgets vary across model scales and are not reliably controlled by reference-free objectives. The non-monotonic pattern (1B and 6.9B trained more stably than 2.8B) is striking, though single-seed runs cannot fully separate it from run-to-run variance.

\paragraph{KL budget and optimization dynamics.}

To better understand these failures, we analyze KL divergence during training. Well-tuned KL-constrained RLHF exhibits an initial growth phase followed by slow KL increase, keeping the policy close to the reference model. This behavior is sensitive to the KL coefficient, but provides a reliable mechanism for controlling optimization.

In contrast, Maximum Entropy RL exhibits fragile KL behavior across model scales. For Pythia~1B and 6.9B, KL divergence grows approximately linearly without immediate collapse, whereas for Pythia~2.8B similar KL magnitudes correspond to severe overoptimization. This indicates that effective KL budgets are inherently model-dependent, and that reference-free objectives lack a mechanism to infer or enforce an appropriate optimization budget. As a result, their success is not predictable across model scales.

\paragraph{KL update magnitudes and optimization stability.}

To disentangle whether overoptimization arises from the objective or from PPO-specific dynamics, we examine the KL divergence between consecutive policy updates. Figures~\ref{fig:conseckl},~\ref{fig:pythia6.9B-1e-6-resuts},~\ref{fig:pythia6.9b-1e-7} show that PPO is not the culprit: in KL-constrained runs, PPO reliably maintains bounded KL updates between successive policies, even in regimes that eventually overoptimize when the KL coefficient is relaxed. In contrast, Maximum Entropy runs consistently exhibit increasing KL drift between updates, including in runs that appear well-behaved early in training, and this effect intensifies despite substantially lower learning rates.

These observations indicate the instability is objective-driven, and the underlying asymmetry is structural. Both shaping terms depend on $\pi_\theta$, but they behave differently when the policy drifts: as $\pi_\theta$ deviates from the reference, the KL term grows and pulls back toward $\piref$, whereas the entropy term \emph{increases} the effective reward whenever $\pi_\theta(y|x)$ shrinks, providing no analogous restoring force. Reference-free objectives therefore lack the structural safeguards needed to control update magnitudes in online RLHF, and tightening PPO clipping alone cannot compensate: for Pythia~2.8B, even clipping ranges as small as $10^{-4}$ still produced increasing KL drift and eventual overoptimization. Figure~\ref{fig:conseckl} shows this asymmetry directly: with the pipeline held fixed, only the Maximum Entropy runs exhibit escalating consecutive-update KL drift.

\subsubsection{Minimum Entropy RL}\label{sec:minentropy}
Motivated by the link between maximum entropy and overoptimization, and by recent work showing entropy minimization can serve as an effective reward signal for LLM reasoning~\citep{agarwal2025unreasonableeffectivenessentropyminimization}, we adopt an unconventional strategy: \emph{minimizing} entropy by flipping the sign of the entropy term in Equation~\ref{eq:maxent_reward_normalized}. This also serves as a sign-flip ablation that holds the optimizer, learning rate, and pipeline fixed, varying only the sign of the regularizer; if overoptimization in MaxEnt were a generic instability of reward shaping, it would persist under this flip.

For Pythia-1B, Minimum Entropy RL prevents overoptimization and achieves competitive win rates even at the higher learning rate of $1 \times 10^{-6}$, under which Maximum Entropy collapses. For Pythia-2.8B, however, no coefficient we tested produced stable improvement: small values stall learning, while larger values lead to overoptimization comparable to Maximum Entropy. The disappearance of the failure mode at 1B---but not at 2.8B---suggests that the entropy term itself contributes to the instability, while also showing that simply flipping its sign is not a scale-robust fix. More broadly, reference-free methods break down once they leave a healthy KL budget, regardless of the sign of the regularizer. Full formulation, sweep coefficients  are in Appendix~\ref{app:minentropy_curves}.

\subsection{Offline Maximum Entropy RLHF (SimPO)}

Although Maximum Entropy RL fails in online settings, its closed-form offline solution, SimPO, is both effective and empirically competitive. This is not solely due to conservative learning rate---one SimPO configuration for Llama~3~\citep{grattafiori2024llama3herdmodels} uses a higher learning rate than DPO---though controlling the KL divergence remains a universal requirement across alignment methods~\citep{gao2022scalinglawsrewardmodel,rafailov2024scalinglawsrewardmodel}.

A common argument for offline stability is that training samples remain in-distribution, removing the need for explicit OOD regularization. However, the reward model can drift out of distribution during optimization, leading to sub-epoch overoptimization~\citep{azar2023generaltheoreticalparadigmunderstand,rafailov2024scalinglawsrewardmodel}; \citet{huang2025correctingmythosklregularizationdirect} propose $\chi^2$ regularization to inject pessimism, though we were unable to reproduce these results. SimPO is intriguing in this context: it achieves strong results without enforcing pessimism or referencing an external policy.

To better understand this behavior, consider the pairwise reward used in DPO:
\begin{align*}
r(y_w|x) - r(y_l|x)
&= \beta \Bigg(
\log \frac{\pi(y_w|x)}{\pi(y_l|x)}
\nonumber \\
&\quad
- \log \frac{\pi_{\text{ref}}(y_w|x)}{\pi_{\text{ref}}(y_l|x)}
\Bigg),
\end{align*}
where the second term reflects the contribution of the reference model. Because both $y_w$ and $y_l$ are sampled from the reference distribution, this term is typically negative but small, acting as an \emph{adaptive regularizer}. This behavior parallels the role of the target margin in SimPO, with the key distinction that SimPO employs a fixed margin rather than a reference-dependent one~\citep{ahrabian2025practicalanalysishumanalignment}.

\begin{figure}
  \centering
  \includegraphics[width=0.4\textwidth]{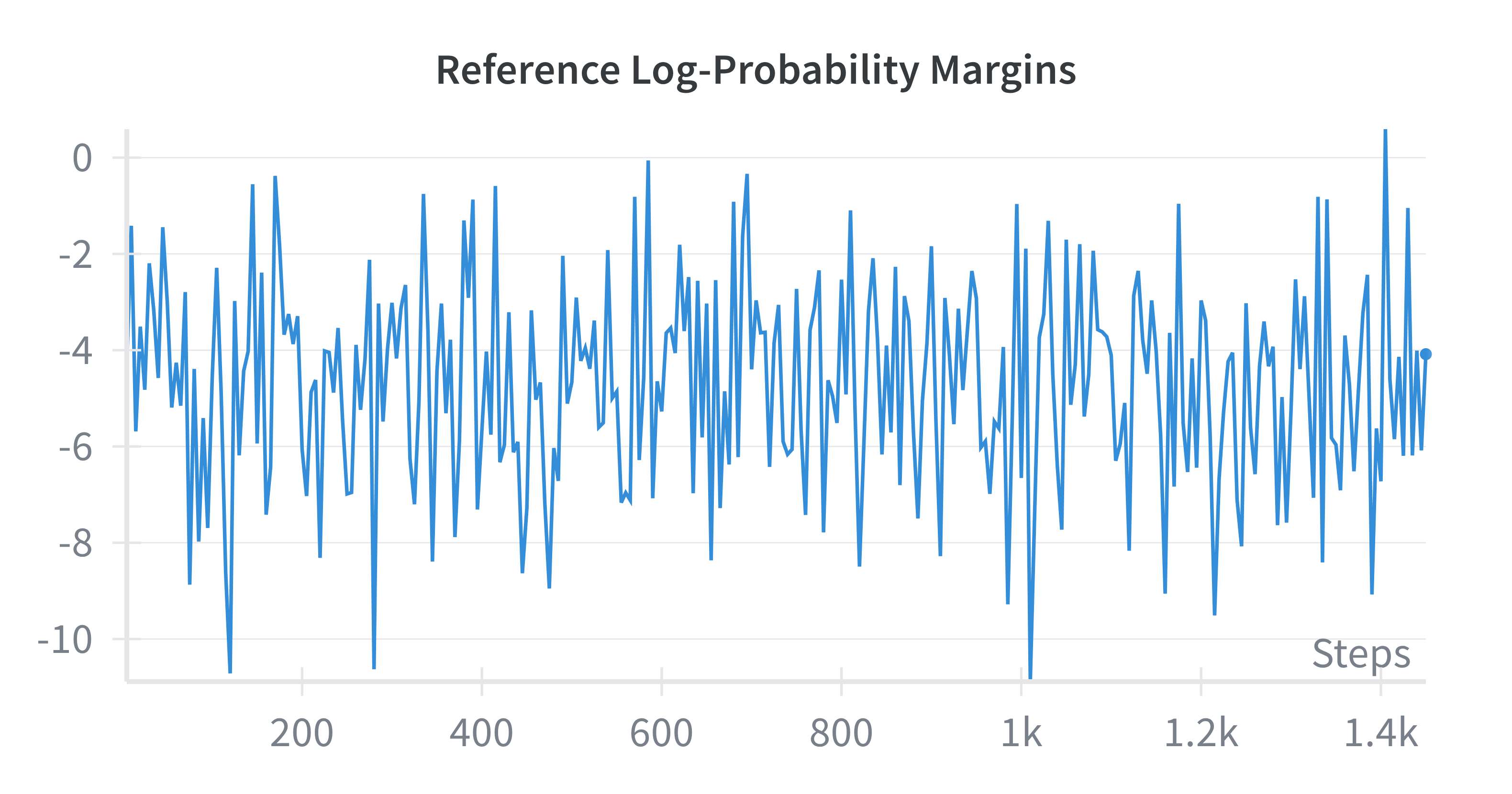}
  \caption{Batch average of
  $
  \log\!\left(
  \frac{\pi_{\text{ref}}(y_w|x)}
       {\pi_{\text{ref}}(y_l|x)}
  \right)
  $
  during DPO training on Pythia~1B.}
  \label{fig:reflog}
\end{figure}

To examine this connection empirically, Figure~\ref{fig:reflog} visualizes the reference log-probability margins $\log\!\left(\pi_{\text{ref}}(y_w|x) / \pi_{\text{ref}}(y_l|x)\right)$ during DPO training with Pythia~1B. They lie within a relatively narrow range, consistent with the fixed margins used in reference-free methods such as SimPO, suggesting that in offline settings reference models may primarily provide a bounded margin signal rather than strong distributional control.

Several caveats apply. Large margins with high learning rates drive aggressive separation, exacerbating overoptimization~\citep{rafailov2024scalinglawsrewardmodel}, consistent with the extreme likelihood decreases we observe. Reference models may implicitly adapt margins across samples (smaller for easy examples, larger for hard), an adaptivity fixed margins lack. Reference models are thus neither necessary nor sufficient to prevent offline overoptimization; stability emerges from the interaction between learning rate, margin, and model capacity (Appendix~\ref{app:simpo_and_margins}).

\section{Conclusion}

This work establishes a theoretical foundation for SimPO by deriving it as a Maximum Entropy RL solution, and uses this connection to expose a sharp offline–online asymmetry. SimPO excels offline, but online Maximum Entropy RL frequently overoptimizes: entropy regularization correlates with reward hacking rather than preventing it, and unlike the KL penalty it provides no restoring force when the policy drifts. SimPO's offline success therefore appears to rest on implicit stabilizers—bounded dataset margins and a fixed target margin—that substitute for the role a reference model plays online, but do not transfer once optimization is no longer constrained to in-distribution samples. Any successful reference-free online method will likely need to recover this stabilizing structure, whether through lightweight margin-based constraints, entropy scheduling, or other mechanisms that implicitly bound the KL budget.

\section{Limitations}\label{sec:limitations}
Our experiments use TL;DR with the Pythia suite. Across three model scales, six entropy coefficients, and 40+ configurations, the findings establish behavior in this controlled setup but not transfer to other tasks (e.g., instruction-following, reasoning) or architectures (e.g., Llama); broader sweeps were prevented by the cost of online RLHF. We use GPT-4o-mini rather than human preferences, though our central claims rest on training metrics rather than win rates. Compute also limited us to single-seed runs; we mitigate this by reporting effects that are large and qualitatively consistent (Figure~\ref{fig:conseckl}; Pythia~2.8B collapsing across all six coefficients), but multi-seed validation would strengthen the claims. Our mechanistic argument is supported by converging evidence (sign-flip ablation, PPO-clipping controls, fixed-pipeline KL comparison) but remains correlational.

\section{Ethical considerations}
This work focuses on the theoretical and empirical analysis of reinforcement learning objectives for aligning large language models. All experiments were conducted on publicly available preference datasets, and no personally identifiable or sensitive information was used. Our results are intended to improve the understanding of alignment methods and do not involve deployment of models in real-world settings. Nevertheless, as with all research on large language models, advances in alignment can have dual-use implications: while they may contribute to safer and more reliable AI systems, they could also lower barriers to developing more capable models that might be misused. We encourage responsible use and further investigation into the societal impacts of alignment research.

\section{The Use of Large Language Models}
All text was initially drafted by the authors, after which Large Language Models were employed to refine phrasing and enhance clarity of expression.

\section*{Acknowledgements}

This work was funded by KUIS AI Center, Koç University.


\bibliography{custom}

\appendix

\section{Additional Related Work}\label{app:related_work}

\paragraph{Reference-free Alignment.}

While early methods like RRHF~\citep{yuan2023rrhfrankresponsesalign} and RAFT~\citep{dong2023raftrewardrankedfinetuning} still relied on external reward models for ranking, they revealed that complex RL dynamics were unnecessary. SLiC-HF~\citep{zhao2023slichfsequencelikelihoodcalibration} showed that sequence likelihood calibration could directly incorporate human feedback without explicit reward modeling. ORPO~\citep{hong2024orpomonolithicpreferenceoptimization} made the key insight that odds ratios could replace probability ratios, enabling monolithic training without reference model drift. CPO~\citep{xu2024contrastivepreferenceoptimizationpushing} and SimPO~\citep{meng2024simposimplepreferenceoptimization} both recognized that sequence probabilities themselves encode preference signals. SimPO can be seen as CPO's length-normalized variant with zero behavior cloning, but this seemingly minor change eliminates the need for hyperparameter tuning of the BC coefficient. The Cringe Loss~\citep{xu2024thingscringeothersiterative} explored iterative self-improvement through token-level soft margins rather than sequence-level optimization. The proliferation of SimPO variants (AlphaPO's~\citep{gupta2025alphaporewardshapematters} reward shaping, $\gamma$PO's adaptive margins~\citep{sun2025robustpreferenceoptimizationdynamic}, AMoPO's~\citep{liu2025amopoadaptivemultiobjectivepreference} multi-objective extension, ConfPO's~\citep{yoon2025confpoexploitingpolicymodel} token-level refinement) demonstrates the flexibility of SimPO's reward formulation while addressing specific optimization challenges.

\paragraph{Maximum Entropy RL Beyond RLHF.}
A broader literature studies entropy regularization in classical RL settings and provides useful context for our findings. Most directly related to our findings,\citet{zhang2025maximumentropymisleadspolicy} show that maximum entropy regularization can mislead policy optimization in continuous control, reaching conclusions that are complementary to ours: in both settings, entropy regularization can fail to deliver the stabilization properties commonly attributed to it. Our work extends this thread to RLHF, where the failure mode is reward hacking against a learned proxy rather than suboptimal exploration in a fixed reward landscape.

\paragraph{Overoptimization in Preference Learning.}
Reward hacking~\citep{skalse2025definingcharacterizingrewardhacking} is a long-standing problem in reinforcement learning~\citep{Sutton1998} where policies achieve high rewards but fail to meet the actual objective~\citep{amodei2016concreteproblemsaisafety,hadfieldmenell2020inverserewarddesign,pan2022effectsrewardmisspecificationmapping}. In language model alignment, this manifests as models learning to generate outputs that score highly on proxy metrics while being of poor actual quality. This overoptimization phenomenon was first systematically studied in traditional RLHF~\citep{christiano2023deepreinforcementlearninghuman,stiennon2022learningsummarizehumanfeedback,gao2022scalinglawsrewardmodel,ouyang2022traininglanguagemodelsfollow}, where optimizing imperfect proxy reward models leads to qualitatively worse outputs, including overly wordy responses and hallucinated information.

Direct alignment algorithms like DPO~\citep{rafailov2024directpreferenceoptimizationlanguage} were designed to bypass RL training by parameterizing rewards directly in terms of the policy, but they introduce their own form of overoptimization. \citet{azar2023generaltheoreticalparadigmunderstand} show that DPO's unbounded log-odds transformation leads to severely overfitted implicit rewards, losing the regularization benefits of standard RLHF's explicit reward modeling. They propose IPO using bounded $\Psi$ functions to address this issue. However, \citet{rafailov2024scalinglawsrewardmodel} demonstrate that even IPO, despite its theoretical guarantees against overoptimization, still exhibits similar degradation patterns to DPO and RLHF at higher KL budgets and across different model scales, suggesting that overoptimization in direct alignment algorithms may be a more fundamental issue than initially anticipated. More recently, \citet{huang2025correctingmythosklregularizationdirect} propose $\chi^2$-Preference Optimization ($\chi$PO), which replaces DPO's logarithmic link function with $\chi^2$-divergence regularization to implement pessimism under uncertainty, providing theoretical guarantees against overoptimization based on single-policy concentrability.

\section{Hyperparameter Settings}\label{app:hyperparameters}

Table~\ref{tab:hyperparameters} reports the complete set of hyperparameters used across all experiments. Training follows \citet{huang2024nimplementationdetailsrlhf} as implemented in TRL~\citep{vonwerra2022trl}, with minimal modifications to support reference-free reward shaping. Reward models are SFT-initialized linear-head models with output normalization. We use greedy decoding for evaluation unless otherwise stated.

\begin{table}[h]
\centering
\tiny
\begin{tabular}{ll}
\toprule
\textbf{Setting} & \textbf{Value} \\
\midrule
Optimizer & AdamW \\
Adam $(\beta_1, \beta_2)$ & $(0.9, 0.999)$ \\
Weight decay & $0.0$ \\
LR schedule & Linear decay \\
Warmup steps & $0$ \\
Total training steps & $\sim$ 1 epoch over TL;DR \\
Rollout batch size (1B) & 64 \\
Rollout batch size (2.8B) & 64 \\
Rollout batch size (6.9B) & 64 \\
PPO mini-batch size & 16 \\
PPO epochs / batch & 4 \\
Generation max length & 53 tokens \\
Generation temperature & 0.7 \\
PPO clip range & 0.2 (default), $10^{-4}$ in clipping ablation \\
Value clip range & 0.2 \\
KL coefficient $\beta$ (KL-RLHF) & $0.05$ \\
Entropy coefficient $\beta$ (MaxEnt) & $\{0.01, 0.05, 0.1, 0.2, 0.3, 0.5\}$ \\
Entropy coefficient $\beta$ (MinEnt) & $\{0.01, 0.05, 0.1, 0.2, 0.3, 0.5\}$ \\
LR (KL-RLHF, all scales) & $1\times 10^{-6}$ \\
LR (MaxEnt, 1B) & $\{1\times 10^{-6}, 1\times 10^{-7}\}$ \\
LR (MaxEnt, 2.8B) & $\{1\times 10^{-7}, 5\times 10^{-8}\}$ \\
LR (MaxEnt, 6.9B) & $\{1\times 10^{-6}, 1\times 10^{-7}\}$ \\
Reward normalization & Yes (running mean/std) \\
Length normalization (MaxEnt) & Yes ($\beta/|y|$) \\
Length normalization (KL) & No \\
Evaluator & GPT-4o-mini \\
Decoding for evaluation & Greedy \\
Seed & Single seed per config \\
\bottomrule
\end{tabular}
\caption{Hyperparameter settings for all experiments. Coefficients of the form $\beta/|y|$ denote length-normalized rewards as defined in Equation~\ref{eq:maxent_reward_normalized}.}
\label{tab:hyperparameters}
\end{table}

In total, the empirical study spans more than 40 training configurations across model scales, learning rates, and entropy coefficients. Each configuration corresponds to a full online RLHF training run with rollout generation at every step.

\section{Margins and Overoptimization}\label{app:simpo_and_margins}

It has been shown that methods such as SimPO can achieve performance comparable to DPO even with a target margin of $\gamma=0$, as demonstrated in the original SimPO paper. This suggests that offline methods do not necessarily require reference models when operating within the safe KL region, and that introducing margins generally improves performance across benchmarks. This effect arises from both model capabilities and dataset coverage: larger models are less prone to common overfitting behaviors and can extract more meaningful signals during optimization, rather than engaging in reward hacking, a phenomenon observed in both online and offline preference optimization~\citep{gao2022scalinglawsrewardmodel,rafailov2024directpreferenceoptimizationlanguage}. Consequently, the influence of the reference model is minimal and can often be neglected. However, this behavior is contingent on the task being sufficiently challenging and the model being strong enough to avoid overoptimization. To validate this observation, we train Pythia-1B on TL;DR using SimPO across different margin values ($\gamma$) and learning rates, in a setting where the model is relatively weaker and the task is easier compared to standard chat datasets such as UltraFeedback~\citep{cui2024ultrafeedbackboostinglanguagemodels} used in SimPO.

We first consider a learning rate of $1 \times 10^{-6}$, which is known to be effective for DPO; DPO metrics are shown in Figure~\ref{fig:dpo}. In this setting, all SimPO models exhibit overoptimization regardless of the $\gamma$ hyperparameter; SimPO metrics are shown in Figure~\ref{fig:simpo1e-6}. Although reward definitions differ and direct comparison of losses or other training metrics is challenging, log-probabilities of samples remain comparable. We observe the characteristic extreme likelihood decreases, which correlate with overoptimization; this pattern is present in DAAs and, as we show, also occurs in online methods. Increasing the margin exacerbates this issue, as optimization aggressively seeks high separation, naturally resulting in overoptimization.

Reference-free methods like SimPO are particularly susceptible because they lack prior knowledge about sample difficulty, treating all samples equally. Some samples are inherently harder and should receive more attention, a behavior that could be partially captured by negative reference contributions in pairwise preferences. When a hardcoded margin pushes the model to satisfy strict separation objectives, it can amplify pathological behaviors during training.

However, when using a relatively low learning rate that allows for gradual updates, SimPO performs significantly better; metrics are shown in Figure~\ref{fig:simpo2e-7} and win rates in Figure~\ref{fig:simpowinrates}. In this regime, it emerges as a strong preference optimization method: an appropriate margin encourages the model to learn and optimize meaningful signals. Therefore, reference-free models require extra safeguards against overoptimization. Controlling the learning rate can act as an anchor, keeping updates within meaningful distributional shifts, although these models can still experience the overoptimization patterns observed in DAAs.

\section{Minimum Entropy RL: Additional Results}\label{app:minentropy_curves}

We provide additional details on the Minimum Entropy RL experiments described in Section~\ref{sec:minentropy}. The Minimum Entropy reward used flips the sign of the entropy term in Equation~\ref{eq:maxent_reward_normalized}, yielding
\begin{equation}\label{eq:minentropy_reward}
r(x, y) = r_\phi(x, y) + \frac{\beta}{|y|} \log \pi_\theta(y|x).
\end{equation}
We sweep $\beta \in \{0.01, 0.05, 0.1, 0.2, 0.3, 0.5\}$ at the same learning rates used for the Maximum Entropy experiments ($1 \times 10^{-6}$ for Pythia~1B, and both $1 \times 10^{-7}$ and $5 \times 10^{-8}$ for Pythia~2.8B). All other settings (optimizer, batch size, generation parameters, reward normalization) match Appendix~\ref{app:hyperparameters}.

\paragraph{Pythia-1B.} At learning rate $1 \times 10^{-6}$ (the rate at which Maximum Entropy variants overoptimize), Minimum Entropy RL trains stably for $\beta \in \{0.05, 0.1\}$, with $\beta = 0.05$ achieving win rates competitive with KL-constrained RLHF and KL divergence remaining bounded throughout training. Smaller coefficients ($\beta = 0.01$) leave optimization too unconstrained; larger coefficients ($\beta \geq 0.2$) are overly conservative and stall improvement.

\paragraph{Pythia-2.8B.} Across both learning rates ($1 \times 10^{-7}$ and $5 \times 10^{-8}$), no Minimum Entropy configuration we tested produced stable win-rate improvement over SFT. Two regimes emerge: small coefficients ($\beta \leq 0.05$) constrain the policy enough that proxy reward stops improving, with win rates remaining at or below SFT; larger coefficients ($\beta \geq 0.1$) eventually lead to overoptimization, resembling the failure mode of Maximum Entropy variants. We do not rule out that an alternative coefficient/learning-rate combination could recover stability at this scale, but the absence of any consistently stable configuration in our sweep is itself informative: it indicates that flipping the sign of the entropy term is not, on its own, sufficient to make reference-free regularization scale-robust.

\paragraph{Discussion.} While entropy minimization succeeds as a standalone reward for reasoning~\citep{agarwal2025unreasonableeffectivenessentropyminimization}, combining it with preference-based rewards creates optimization instabilities that scale poorly with model size. Reducing the learning rate further might recover some stability, but Minimum Entropy RL is not a one-to-one substitute for KL regularization, which remains more dynamic and adaptive across scales.

\newpage
\section{Mathematical Derivations for Maximum Entropy RL}\label{app:mathderivations}
\subsection{Deriving the Optimum of the Entropy-Regularized Reward Maximization Objective}
In this appendix, we will derive the optimal policy for Maximum Entropy RL. Analogously to the KL-constrained case~\citep{rafailov2024directpreferenceoptimizationlanguage}, we optimize the following objective:
\begin{align*}
\max_{\pi}  \mathbb{E}_{x\sim \mathcal{D}, y\sim \pi}\bigl[r(x, y)\bigr] + \alpha\mathcal{H}\bigl[\pi(y|x)\bigr]
\end{align*}
under any reward function $r(x,y)$ and a general non-parametric policy class, where $\mathcal{H}[\pi(y|x)] = -\mathbb{E}_{y\sim \pi(y|x)}[\log \pi(y|x)]$ is the entropy of the policy. We now have:
\label{eq:MaxEnt_RL_proof}
\begin{align*}
&\max_{\pi}  \mathbb{E}_{x\sim \mathcal{D}, y\sim \pi}\bigl[r(x, y)\bigr] + \alpha\mathcal{H}\bigl[\pi(y|x)\bigr] \nonumber\\
&=\max_{\pi}  \mathbb{E}_{x\sim \mathcal{D}}\mathbb{E}_{y\sim \pi(y|x)}\left[r(x, y) - \alpha\log\pi(y|x)\right] \nonumber\\
&=\min_{\pi}  \mathbb{E}_{x\sim \mathcal{D}}\mathbb{E}_{y\sim \pi(y|x)}\left[\log\pi(y|x) - \frac{1}{\alpha}r(x, y)\right] \nonumber\\ 
&=\min_{\pi}  \mathbb{E}_{x\sim \mathcal{D}}\mathbb{E}_{y\sim \pi(y|x)}\Bigg[\log\frac{\pi(y|x)}{\frac{1}{Z(x)}\exp\left(\frac{1}{\alpha}r(x, y)\right)} \nonumber\\
&\qquad\qquad\qquad\qquad - \log Z(x)\Bigg]
\end{align*}

where we have partition function:
\begin{align*}
Z(x) = \sum_{y}\exp\left(\frac{1}{\alpha}r(x, y)\right).
\end{align*}
Note that the partition function is a function of only $x$ and the reward function $r$, but does not depend on the policy $\pi$. We can now define
\begin{align*}
    \pi^*(y|x) = \frac{1}{Z(x)}\exp\left(\frac{1}{\alpha}r(x, y)\right),
\end{align*}

which is a valid probability distribution as $\pi^*(y|x)\geq 0$ for all $y$ and $\sum_{y}\pi^*(y|x)=1$. Since $Z(x)$ is not a function of $y$, we can then re-organize the final objective in Eq.~\ref{eq:MaxEnt_RL_proof} as:

\begin{align*}
&\min_{\pi} \mathbb{E}_{x\sim \mathcal{D}}\bigg[\mathbb{E}_{y\sim \pi(y|x)}\left[\log\frac{\pi(y|x)}{\pi^*(y|x)}\right] \nonumber \\
&\qquad\qquad\qquad - \log Z(x)\bigg]= \\
& \min_{\pi} \mathbb{E}_{x\sim\mathcal{D}} \left[D_{\text{KL}}(\pi(y|x)\|\pi^*(y|x)) - \log Z(x)\right]
\end{align*} 

Since $Z(x)$ is independent of $\pi$, the minimum is attained by the policy that minimizes the first KL term. By Gibbs' inequality, the KL divergence reaches its minimum value of zero if and only if the two distributions are identical. Therefore, this yields the optimal solution:
\label{app:maxent_derivation1}
\begin{equation}
    \pi(y|x)= \pi^*(y|x) = \frac{1}{Z(x)}\exp\left(\frac{1}{\alpha}r(x, y)\right)
\end{equation}
for all $x\in\mathcal{D}$. This completes the derivation.

\subsection{Deriving the SimPO Objective Under the Bradley-Terry Model}
\label{app:maxent_derivation2}
It is straightforward to derive the SimPO objective under the Bradley-Terry preference model as we have
\label{eq:BT_restated_maxent}
\begin{equation}
    p^*(y_1\succ y_2|x)=\frac{\exp\left(r^*(x, y_1)\right)}{\exp\left(r^*(x, y_1)\right) + \exp\left(r^*(x, y_2)\right)}
\end{equation}

We can express the (unavailable) ground-truth reward through its corresponding optimal policy:
\label{eq:main_eq_restated_maxent}
\begin{equation}
    r^*(x,y) =\alpha \log \pi^*(y|x) + \alpha \log Z(x)
\end{equation}
Substituting Eq.~\ref{eq:main_eq_restated_maxent} into Eq.~\ref{eq:BT_restated_maxent} we obtain:

\begin{align*}
p^*(y_1 &\succ y_2|x) \\
&= \frac{\exp(\alpha \log \pi^*(y_1|x))}{\sum_{i \in \{1,2\}} \exp(\alpha \log \pi^*(y_i|x))} \\
&= \frac{1}{1+\exp\left(\alpha \log \frac{\pi^*(y_2|x)}{\pi^*(y_1|x)}\right)} \\
&= \sigma\left(\alpha \log \pi^*(y_1|x) - \alpha \log \pi^*(y_2|x)\right)
\end{align*}

The last line is the per-instance loss for SimPO, without target margin $\gamma$ and length normalization.

\subsection{Deriving the SimPO Objective Under the Plackett-Luce Model}
\label{app:maxent_plackett_luce_models}
The Plackett-Luce model \citep{af5079a1-8ca5-3727-a405-0a82390327b7} extends the Bradley-Terry model from pairwise comparisons to full rankings. As in the Bradley-Terry framework, the probability of selecting an option is assumed to be proportional to the value of an underlying latent reward function. In our setting, given a prompt $x$ and a collection of $K$ candidate answers $y_1, \ldots, y_K$, the user produces a permutation $\tau:[K]\to[K]$ that represents their ranking of the answers. Under the Plackett-Luce model, the probability of such a ranking is defined as follows. Let $r_k = r^*(x, y_{\tau(k)})$. Then:
\begin{equation}\label{eq:pl-model-maxent}
    p^*(\tau| y_1,\ldots, y_K, x)= \prod_{k=1}^{K}\frac{\exp(r_k)}{\sum_{j=k}^{K}\exp(r_j)}
\end{equation}
Observe that when $K=2$, Equation~\ref{eq:pl-model-maxent} simplifies to the Bradley-Terry model. For the general Plackett-Luce model, however, we can still leverage the reward parameterization by substituting the reward function expressed in terms of its optimal policy. As in Appendix~\ref{app:maxent_derivation2}, the normalization constant $Z(x)$ cancels out. Let $s_k = \alpha \log \pi^*(y_{\tau(k)}|x)$. Then:
\begin{equation}
    p^*(\tau| y_1,\ldots, y_K, x)= \prod_{k=1}^{K}\frac{\exp(s_k)}{\sum_{j=k}^{K}\exp(s_j)}
\end{equation}

Similarly to the approach for standard DPO, if we have access to a dataset $\mathcal{D} = \{\tau^{(i)}, y_1^{(i)}, \ldots, y_K^{(i)}, x^{(i)}\}_{i=1}^N$ of prompts and user-specified rankings, we can use a parameterized model and optimize this objective with maximum-likelihood. Let $s_k^\theta = \alpha \log \pi_\theta(y_{\tau(k)}|x)$. Then:
\begin{align*}
    \mathcal{L}_{\text{SimPO}}(\pi_{\theta}) \\ = -\mathbb{E}_{\mathcal{D}}\left[\sum_{k=1}^{K}\log\frac{\exp(s_k^\theta)}{\sum_{j=k}^{K}\exp(s_j^\theta)}\right]
\end{align*}

\subsection{Deriving the Gradient of the SimPO Objective}
\label{app:maxent_gradient_derivation}
In this section we derive the gradient of the SimPO objective:
\label{eq:maxent-grad-start}
\begin{align*}
    \nabla_{\theta}\mathcal{L}_\text{SimPO}(\pi_{\theta}) \\
    = -\nabla_{\theta}\mathbb{E}_{(x, y_w, y_l)\sim \mathcal{D}} \\\left[\log \sigma \left(\alpha \log \pi_{\theta}(y_w|x) - \alpha \log \pi_{\theta}(y_l|x)\right)\right]
\end{align*}

We can rewrite the RHS of Equation~\ref{eq:maxent-grad-start} as 
\begin{align*}
    \nabla_{\theta}\mathcal{L}_\text{SimPO}(\pi_{\theta})
    =-\mathbb{E}_{(x, y_w, y_l)\sim \mathcal{D}}\left[\frac{\sigma'\left(u\right)}{\sigma \left(u\right)}\nabla_{\theta}\left(u\right)\right],
\end{align*}
where $u = \alpha \log \pi_{\theta}(y_w|x) - \alpha \log \pi_{\theta}(y_l|x)$.

Using the properties of sigmoid function $\sigma'(x) = \sigma(x)(1-\sigma(x))$ and $\sigma(-x) = 1-\sigma(x)$, we obtain the final gradient

\begin{multline*}
\nabla_{\theta}\mathcal{L}_\text{SimPO}(\pi_{\theta}) = 
     -\mathbb{E}_{(x, y_w, y_l) \sim \mathcal{D}} \\ \bigg[\alpha\sigma \left(\alpha \log \pi_{\theta}(y_l|x) - \alpha \log \pi_{\theta}(y_w|x)\right) \\ \bigg[\nabla_\theta\log \pi(y_w \mid x) - \nabla_\theta\log\pi(y_l \mid x)\bigg]\bigg],
\end{multline*}

After using the reward substitution of $\hat{r}_\theta(x, y) = \alpha \log \pi_\theta(y \mid x)$ we obtain the final form of the gradient.

\subsection{Proof of Lemma 1 and 2 from DPO for Maximum Entropy RL}
\label{app:maxent_lemma1}

In this section, we will prove the two lemmas from DPO for Maximum Entropy RL.

\begin{lemma}[Lemma 1]
Under the Plackett-Luce preference framework, and in particular the Bradley-Terry framework, two reward functions from the same equivalence class induce the same preference distribution.
\end{lemma}
\begin{proof}
We say that two reward functions $r(x, y)$ and $r'(x, y)$ are from the same equivalence class if $r'(x, y) = r(x, y) + f(x)$ for some function $f$. We consider the general Plackett-Luce (with the Bradley-Terry model a special case for $K=2$) and denote the probability distribution over rankings induced by a particular reward function $r(x, y)$ as $p_r$. For any prompt $x$, answers $y_1,\ldots, y_K$ and ranking $\tau$, let $r_k = r(x, y_{\tau(k)})$ and $r'_k = r'(x, y_{\tau(k)}) = r_k + f(x)$. Then:
\begin{align*}
    p_{r'}&(\tau| y_1,\ldots, y_K, x) \\
    &= \prod_{k=1}^{K}\frac{\exp(r'_k)}{\sum_{j=k}^{K}\exp(r'_j)} \\
    &= \prod_{k=1}^{K}\frac{\exp(r_k + f(x))}{\sum_{j=k}^{K}\exp(r_j + f(x))} \\
    &= \prod_{k=1}^{K}\frac{\exp(f(x))\exp(r_k)}{\exp(f(x))\sum_{j=k}^{K}\exp(r_j)} \\
    &= \prod_{k=1}^{K}\frac{\exp(r_k)}{\sum_{j=k}^{K}\exp(r_j)} \\
    &= p_{r}(\tau| y_1,\ldots, y_K, x)
\end{align*}
which completes the proof.
\end{proof}

\begin{lemma}[Lemma 2]
Two reward functions from the same equivalence class induce the same optimal policy under the entropy-regularized RL problem.
\end{lemma}
\begin{proof}
Let us consider two reward functions from the same class, such that $r'(x, y)=r(x, y)+f(x)$ and, let us denote as $\pi_r$ and $\pi_{r'}$ the corresponding optimal policies. For all $x, y$, let $\rho = \frac{1}{\alpha}r(x,y)$, $\rho' = \frac{1}{\alpha}r'(x,y)$, and $\phi = \frac{1}{\alpha}f(x)$. Then:
\begin{align*}
    \pi_{r'}(y|x) &= \frac{\exp(\rho')}{\sum_{y}\exp(\rho')} \\
    &= \frac{\exp(\rho + \phi)}{\sum_{y}\exp(\rho + \phi)} \\
    &= \frac{\exp(\phi)\exp(\rho)}{\exp(\phi)\sum_{y}\exp(\rho)} \\
    &= \frac{\exp(\rho)}{\sum_{y}\exp(\rho)} \\
    &= \pi_r(y|x),
\end{align*}
\end{proof}

\subsection{Proof of Theorem 1 from DPO for Maximum Entropy RL}
\label{app:maxent_thm1}

In this section, we will elaborate on the results of the main theorem from DPO for Maximum Entropy RL.

\begin{theorem}[Maximum Entropy Version]
Assume we have a parameter $\alpha>0$. All reward equivalence classes, as defined in the previous section, can be represented with the reparameterization $r(x, y) = \alpha \log \pi(y|x)$ for some model $\pi(y|x)$.
\end{theorem}
\begin{proof}
Consider any reward function $r(x,y)$, which induces an optimal model $\pi_r(y|x)$ under the entropy-regularized RL problem, with solution given by the optimal policy derivation. We have:
\begin{equation*}
    r(x,y) =\alpha \log \pi_r(y|x) + \alpha \log Z(x)
\end{equation*}
where $Z(x) =\sum_{y}\exp\left(\frac{1}{\alpha}r(x, y)\right)$ (notice that $Z(x)$ also depends on the reward function $r$). Using the operator $r'(x, y) = f(r, \alpha)(x, y) = r(x, y) - \alpha \log Z(x)$, we see that this new reward function is within the equivalence class of $r$ and, we have:
\begin{equation*}
    r'(x,y) =\alpha \log \pi_r(y|x)
\end{equation*}

which completes the proof.
\end{proof}
We can further expand on these results. We can see that if $r$ and $r'$ are two reward functions in the same class, then
\begin{align*}
    f(r, \alpha)(x, y)= \alpha \log \pi_r(y|x)= \\
\alpha \log \pi_{r'}(y|x) = f(r', \alpha)(x, y)
\end{align*}
where the second equality follows from Lemma 2. We have proven that the operator $f$ maps all reward functions from a particular equivalence class to the same reward function. Next, we show that for every equivalence class of reward functions, the reward function that has the reparameterization outlined in the main theorem is unique.

\begin{proposition}\label{prop:maxent_unique}
Assume we have a parameter $\alpha>0$. Then every equivalence class of reward functions has a unique reward function $r(x, y)$, which can be reparameterized as $r(x, y) = \alpha \log \pi(y|x)$ for some model $\pi(y|x)$.
\end{proposition}
\begin{proof}
    We will proceed using proof by contradiction. Assume we have two reward functions from the same class, such that $r'(x, y) = r(x, y) + f(x)$. Moreover, assume that  $r'(x, y) = \alpha \log \pi'(y|x)$ for some model $\pi'(y|x)$ and  $r(x, y) = \alpha \log \pi(y|x)$ for some model $\pi(y|x)$, such that $\pi\neq\pi'$. We then have
\begin{align*}
r'(x, y) = r(x, y) + f(x)  = \alpha \log \pi(y|x) + f(x) \\ =\alpha \log \pi(y|x)\exp(\tfrac{1}{\alpha} f(x))  \\ = \alpha \log \pi'(y|x)
\end{align*}

    for all prompts $x$ and completions $y$. Then we must have $\pi(y|x)\exp(\tfrac{1}{\alpha} f(x)) = \pi'(y|x)$. Since these are distributions, summing over $y$ on both sides, we obtain that $\exp(\tfrac{1}{\alpha} f(x)) = 1$ and since $\alpha>0$, we must have $f(x)=0$ for all $x$. Therefore $r(x,y) = r'(x,y)$. This completes the proof.
\end{proof}

We have now shown that every reward class has a unique reward function that can be represented as outlined in the main theorem, which is given by $f(r, \alpha)$ for any reward function in that class.

\section{Extra Figures}
\begin{figure*}[h]
    \centering
    \includegraphics[width=0.9\textwidth]{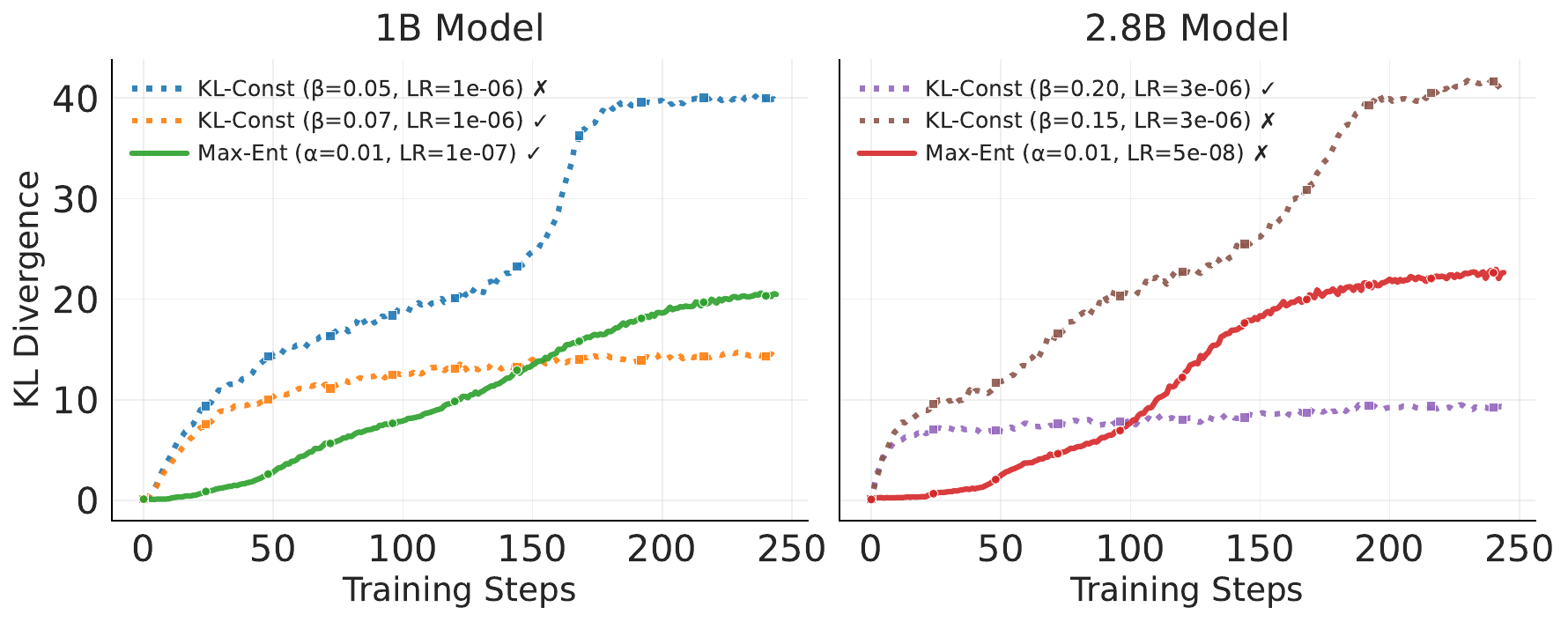}
    \caption{KL divergence evolution during training for 1B and 2.8B parameter models using different regularization methods. The left panel shows results for the 1B model and the right panel shows results for the 2.8B model. Each panel compares KL-Constrained and Maximum-Entropy approaches. Checkmarks ($\checkmark$) indicate high win rate runs and crosses ($\times$) indicate overoptimized runs.}

    \label{fig:KLdiv}
\end{figure*}

\begin{figure*}[h]
    \centering
    \includegraphics[width=0.6\textwidth]{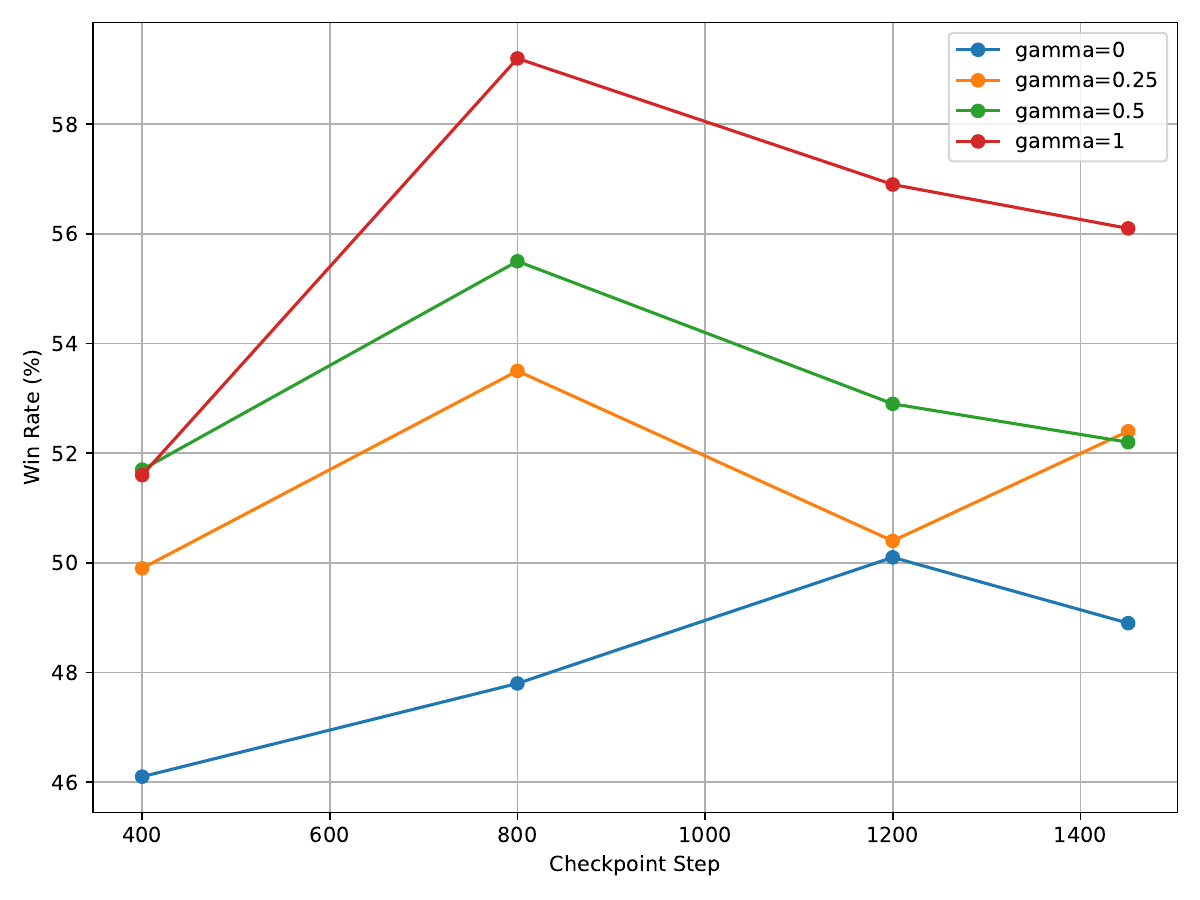}
    \caption{Win rate progression across training checkpoints for different values of the gamma hyperparameter. Results are for the Pythia-1B model trained with a learning rate of $2\times10^{-7}$.}

    \label{fig:simpowinrates}
\end{figure*}

\begin{figure*}[h]
    \centering
    \includegraphics[width=1\textwidth]{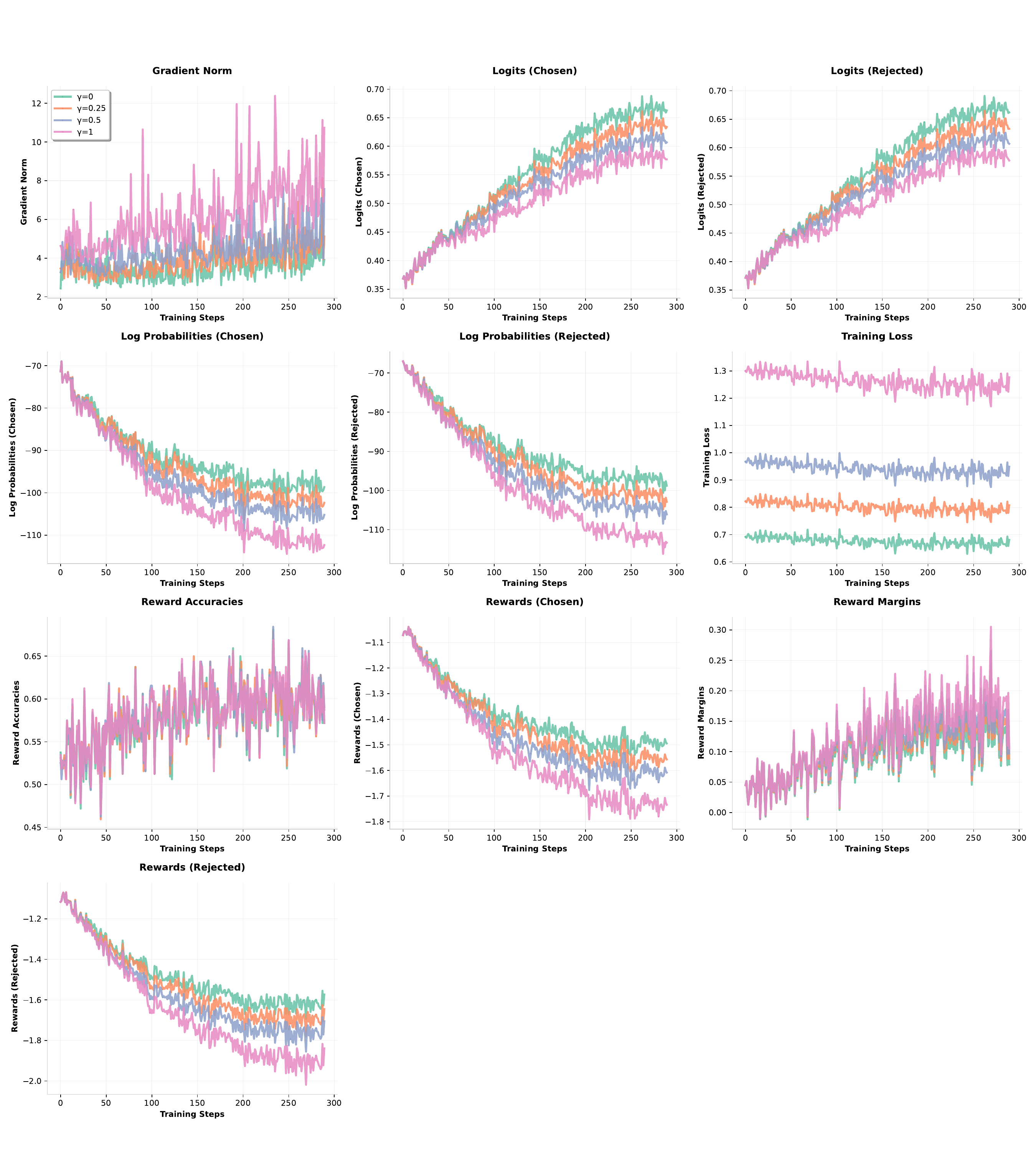}
    \caption{SimPO training metrics across different gamma values. Comparison of key training dynamics including loss, gradients, logits, and reward metrics for $\gamma \in \{0,0.25,0.5,1.0\}$ using Pythia-1B with learning rate $2 \times 10^{-7}$.}

    \label{fig:simpo2e-7}
\end{figure*}

\begin{figure*}[h]
    \centering
    \includegraphics[width=1\textwidth]{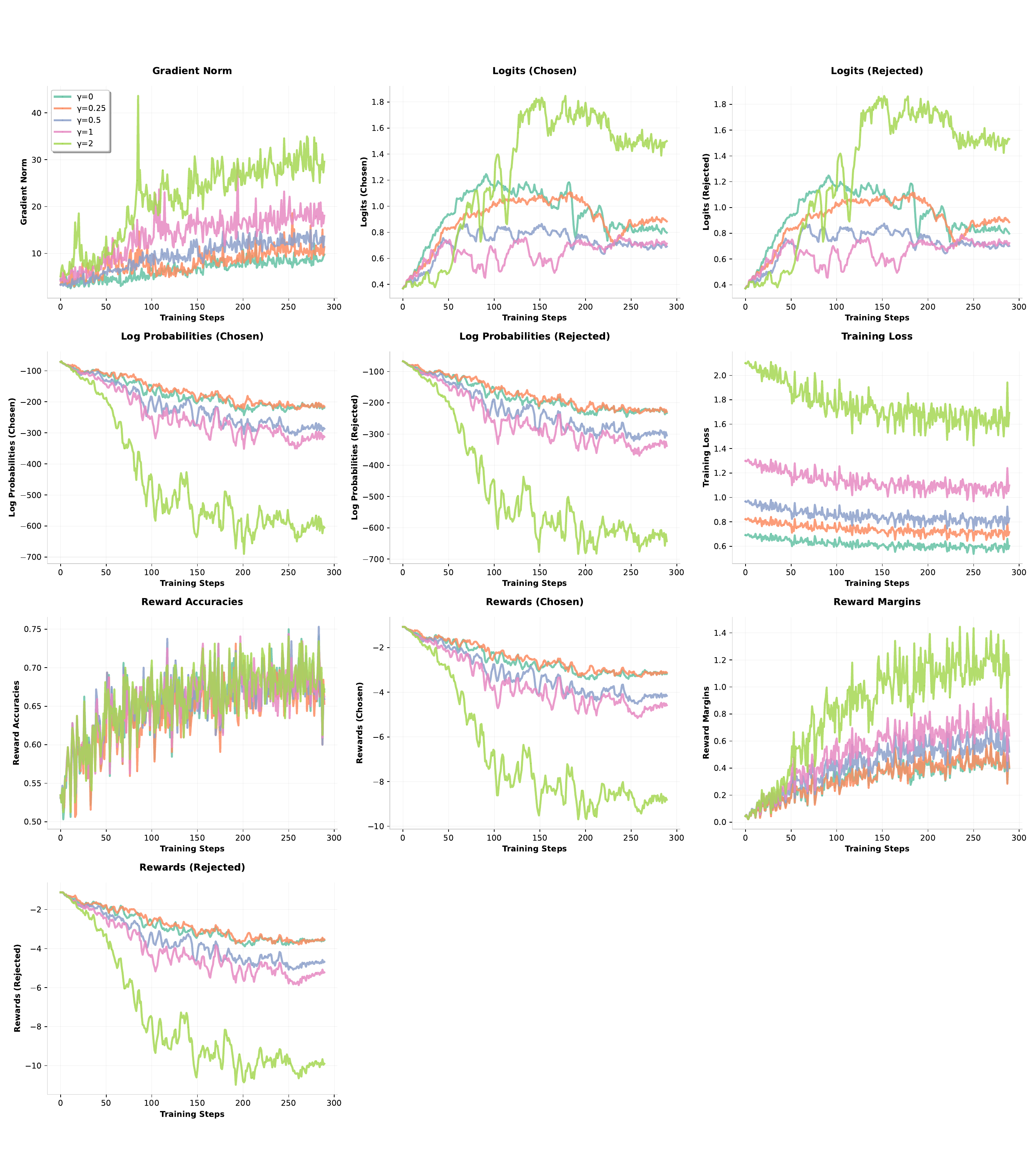}
    \caption{SimPO training metrics across different gamma values. Comparison of key training dynamics including loss, gradients, logits, and reward metrics for $\gamma \in \{0,0.25,0.5,1.0,2.0\}$ using Pythia-1B with learning rate $1 \times 10^{-6}$.}

    \label{fig:simpo1e-6}
\end{figure*}

\begin{figure*}[h]
    \centering
    \includegraphics[width=1\textwidth]{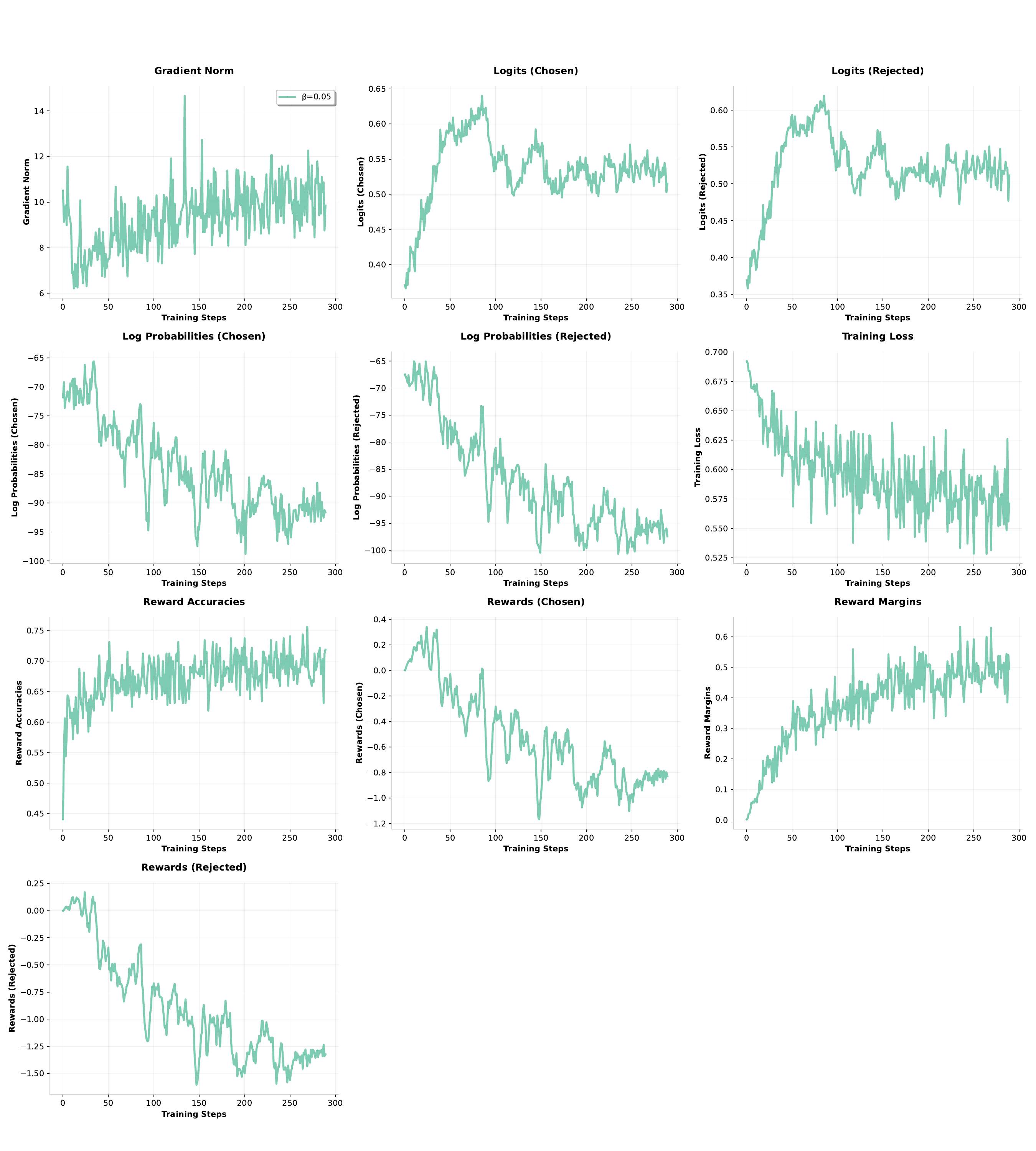}
    \caption{DPO training metrics with $\beta=0.05$. Comparison of key training dynamics including loss, gradients, logits, and reward metrics, using Pythia-1B with learning rate $1 \times 10^{-6}$.}

    \label{fig:dpo}
\end{figure*}

\begin{figure*}[h]
    \centering
    \includegraphics[width=0.7\textwidth]{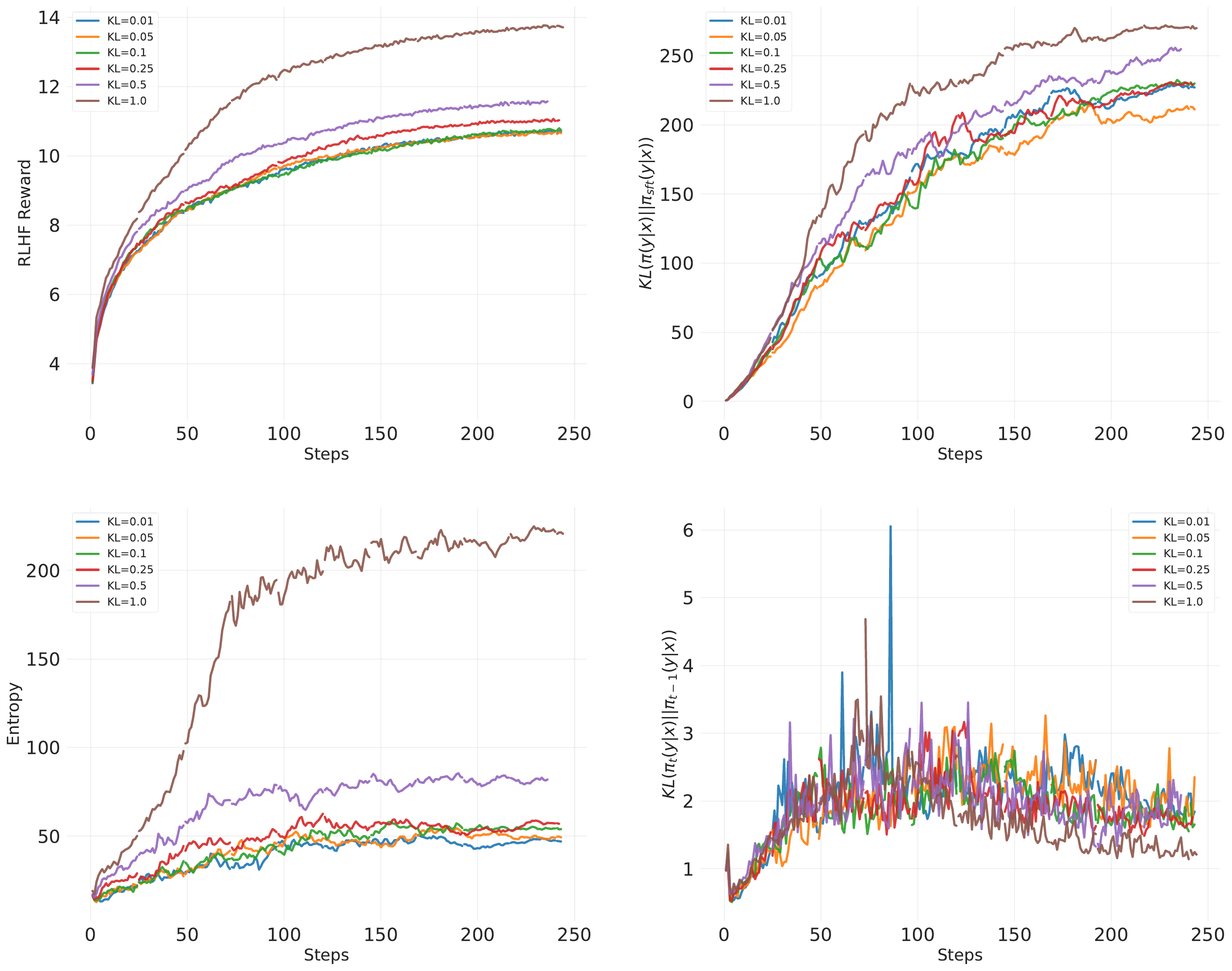}
    \caption{Reward dynamics and KL divergence metrics for entropy-regularized RL training across different entropy coefficients. Top-left panel shows reward progression (RLHF reward) over training steps for various entropy values. Top-right panel shows KL divergence between the current policy and the SFT reference policy ($KL(\pi_t || \pi_\text{SFT})$). Bottom-left panel tracks entropy reward across training steps. Bottom-right panel displays KL divergence between consecutive policy updates ($KL(\pi_t || \pi_{t-1})$). All plots are based on the Pythia-6.9B model trained with the learning rate of $1 \times 10^{-6}$. }

    \label{fig:pythia6.9B-1e-6-resuts}
\end{figure*}

\begin{figure*}[h]
    \centering
    \includegraphics[width=0.7\textwidth]{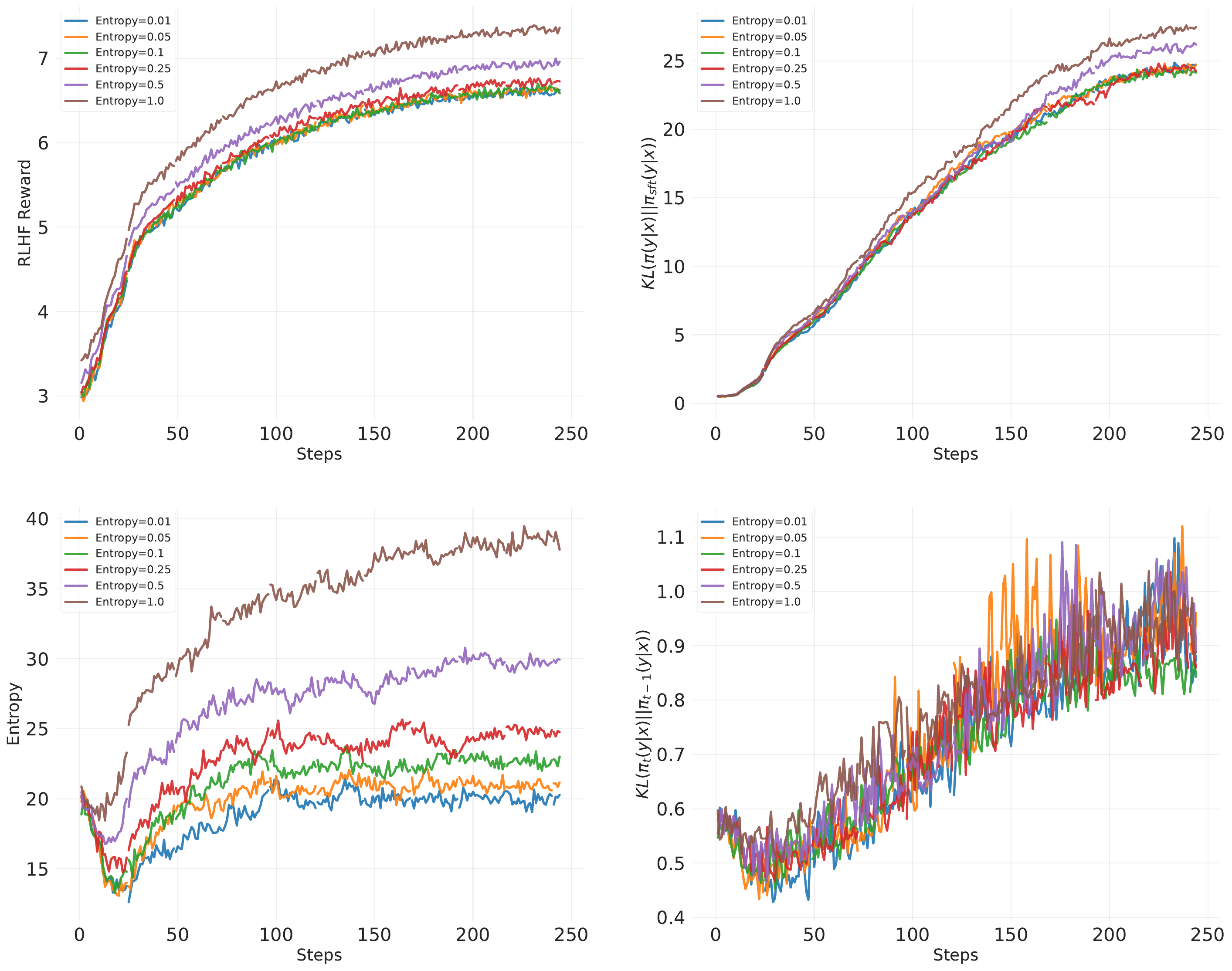}
    \caption{Reward dynamics and KL divergence metrics for entropy-regularized RL training across different entropy coefficients. Top-left panel shows reward progression (RLHF reward) over training steps for various entropy values. Top-right panel shows KL divergence between the current policy and the SFT reference policy ($KL(\pi_t || \pi_\text{SFT})$). Bottom-left panel tracks entropy reward across training steps. Bottom-right panel displays KL divergence between consecutive policy updates ($KL(\pi_t || \pi_{t-1})$). All plots are based on the Pythia-6.9B model trained with the learning rate of $1 \times 10^{-7}$. }

    \label{fig:pythia6.9b-1e-7}
\end{figure*}

\begin{figure*}[h]
    \centering
    \includegraphics[width=0.8\textwidth]{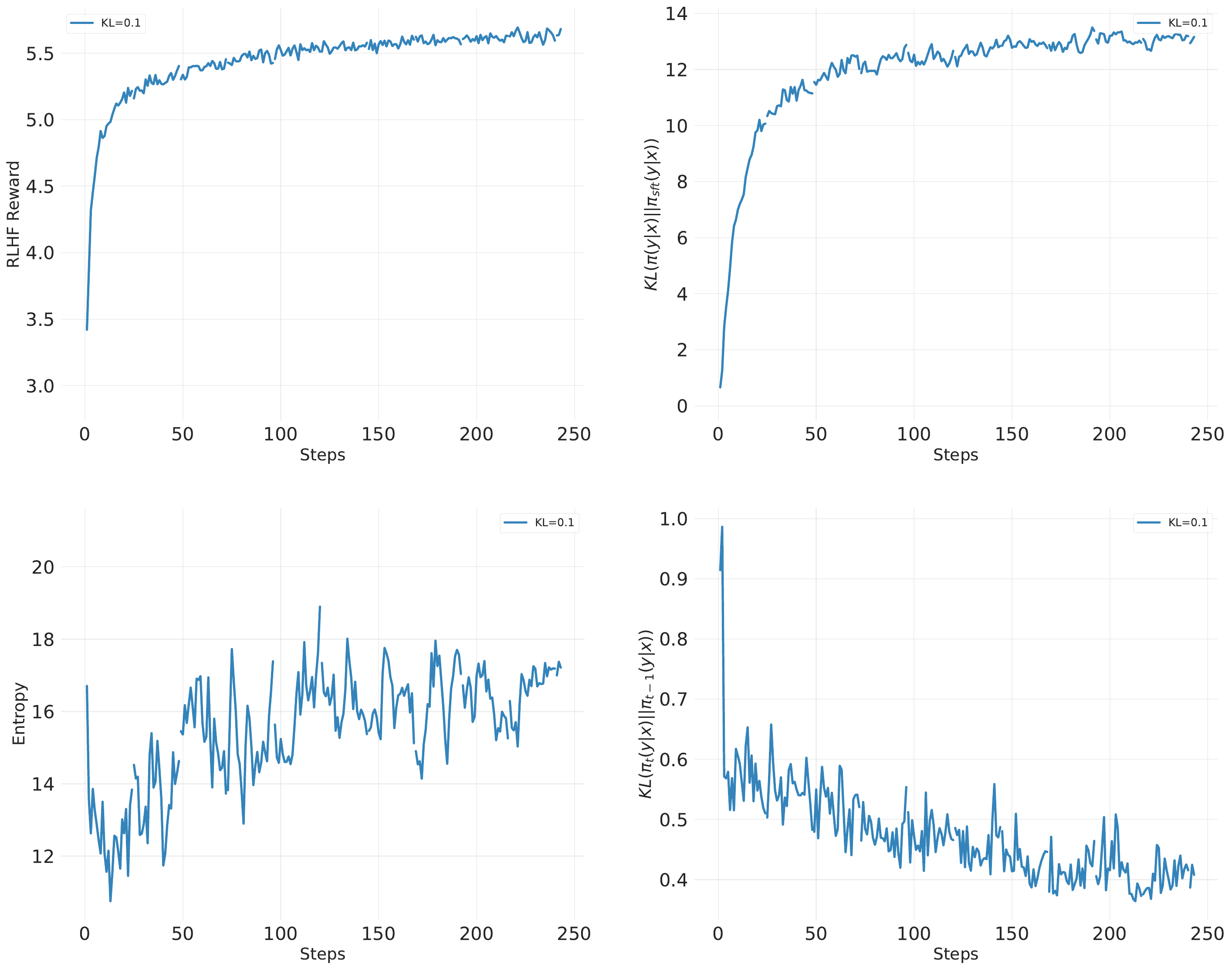}
    \caption{Training metrics for KL-constrained RL on the Pythia-6.9B model. Top panel shows the KL divergence between the policy and reference SFT policy ($KL(\pi_t || \pi_\text{SFT})$) over training steps. Middle panel displays the reward trajectory (RLHF reward). Bottom panel shows the KL divergence between consecutive policy updates ($KL(\pi_t || \pi_{t-1})$). All results correspond to a single training run with a fixed KL constraint.}

    \label{fig:pythia6.9b-kl}
\end{figure*}

\end{document}